\documentclass[twoside]{article}

\usepackage[accepted]{aistats2021}
\usepackage[english]{babel}
\usepackage[utf8]{inputenc} % allow utf-8 input
\usepackage[T1]{fontenc}    % use 8-bit T1 fonts
\usepackage{hyperref}       % hyperlinks
\usepackage{lmodern}
\usepackage{url}            % simple URL typesetting
\usepackage{booktabs}       % professional-quality tables
\usepackage{amsfonts}       % blackboard math symbols
\usepackage{nicefrac}       % compact symbols for 1/2, etc.
\usepackage{microtype}      % microtypography
\usepackage{enumitem}
\usepackage{algorithm}
\usepackage{algorithmic}
\usepackage{graphicx}
\usepackage[usenames,dvipsnames]{xcolor} %%need to get rid of 100+ warnings
\usepackage[titletoc,title]{appendix}
\usepackage{amsmath,amsthm,amssymb,mathtools}
\usepackage{braket}
\usepackage{dsfont}
\usepackage{natbib}
\usepackage[final,inline,nomargin]{fixme}
\usepackage{thmtools,thm-restate}
\usepackage{subcaption}
\usepackage{tikz}
\usetikzlibrary{positioning}

\newtheorem{theorem}{Theorem}[section]
\newtheorem*{theorem*}{Theorem}

\newtheorem*{proposition*}{Proposition}
\newtheorem{lemma}[theorem]{Lemma}
\newtheorem*{lemma*}{Lemma}

\newtheorem*{conjecture*}{Conjecture}

\newtheorem*{fact*}{Fact}

\newtheorem*{hypothesis*}{Hypothesis}

\newtheorem{claim}[theorem]{Claim}
\newtheorem*{claim*}{Claim}

\theoremstyle{definition}
\newtheorem{definition}[theorem]{Definition}

\theoremstyle{remark}

\newtheorem*{remark*}{Remark}

\newtheorem*{observation*}{Observation}

\newcommand{\indicator}{\mathds{1}}

\newcommand{\R}{\mathbb{R}}

\newcommand{\abs}[1]{|#1|}

\newcommand{\calE}{\mathcal{E}}

\newcommand{\calL}{\mathcal{L}}

\newcommand{\calR}{\mathcal{R}}

\newcommand{\norm}[1]{\lVert #1 \rVert}

\newcommand{\Esymb}{\mathbb{E}}

\DeclareMathOperator*{\E}{\Esymb}

\DeclareMathOperator*{\argmax}{\text{argmax}}
\DeclareMathOperator*{\argmin}{\text{argmin}}
\renewcommand{\Pr}{\ProbOp}

\newcommand{\prob}[1]{\Pr\set{#1}}

\newcommand{\tw}{\tilde{w}}

\newcommand{\tc}{\tilde{c}}

\newcommand{\hc}{\hat{c}}
\newcommand{\hw}{\hat{w}}
\newcommand{\hx}{\hat{x}}
\newcommand{\bw}{\bar{w}}
\newcommand{\bc}{\bar{c}}
\newcommand{\bx}{\bar{x}}
\newcommand{\obsins}{(G,\hat{c},\hat{w})}
\newcommand{\stabins}{(G,\bar{c},\bar{w})}

\renewcommand{\epsilon}{\varepsilon}

\def\rbr#1{\left(#1\right)}   %round bracket
\renewcommand{\cite}[1]{\citep{#1}}

\ifdefined\nohyperref\else\ifdefined\hypersetup
  \definecolor{mydarkblue}{rgb}{0,0.08,0.45}
  \hypersetup{ %
    colorlinks=true,
    linkcolor=mydarkblue,
    citecolor=mydarkblue,
    filecolor=mydarkblue,
    urlcolor=mydarkblue,
    pdfview=FitH}
\fi\fi

\newif\ifnotes\notestrue
\ifnotes
\usepackage{color}
\definecolor{mygrey}{gray}{0.50}
\newcommand{\notename}[2]{{\textcolor{red}{\footnotesize{\bf (#1:} {#2}{\bf ) }}}}

\else

\newcommand{\notename}[2]{{}}

\fi

\usepackage{tikz}
\usepackage{multirow}
\usepackage{amsmath}
\usepackage{graphicx}
\newcommand{\maximize}{\mathop{\mathrm{maximize}{}}}
\newcommand{\minimize}{\mathop{\mathrm{minimize}{}}}
\newcommand{\mindot}{\mathop{\mathrm{min.}{}}}
\renewcommand{\Pr}{\mathbb{P}}
\renewcommand{\dot}[2]{\langle #1,#2\rangle}

\long\def\aistatssupptitle#1{
  \hsize\textwidth
  \linewidth\hsize \toptitlebar {\centering
  {\Large\bfseries #1 \par}}
 \bottomtitlebar
}

\begin{document}
\twocolumn[
\aistatstitle{Beyond Perturbation Stability: LP Recovery Guarantees for MAP Inference on Noisy Stable Instances}
\aistatsauthor{Hunter Lang${}^*$ \And Aravind Reddy${}^*$ \And David Sontag \And Aravindan Vijayaraghavan }
\aistatsaddress{ MIT \\ {\tt hjl@mit.edu}\And Northwestern University \\ {\tt arareddy@u.northwestern.edu} \And MIT\\ {\tt dsontag@mit.edu} \And Northwestern University\\ {\tt aravindv@northwestern.edu}} ]

\begin{abstract}
  Several works have shown that \emph{perturbation stable} instances of the MAP inference problem in Potts models can be solved exactly using a natural linear programming (LP) relaxation.
  However, most of these works give few (or no) guarantees for the LP solutions on instances that do not satisfy the relatively strict perturbation stability definitions. 
  In this work, we go beyond these stability results by showing that the LP approximately recovers the MAP solution of a stable instance even after the instance is corrupted by noise.
  This ``noisy stable'' model realistically fits with practical MAP inference problems: we design an algorithm for finding ``close'' stable instances, and show that several real-world instances from computer vision have nearby instances that are perturbation stable. These results suggest a new theoretical explanation for the excellent performance of this LP relaxation in practice.
\end{abstract}
\section{Introduction}
In this work, we study the MAP inference problem in the \emph{ferromagnetic Potts model}, which is also known as uniform metric labeling \citep{KleinbergTardos02}. Given a graph $G=(V,E)$, this problem is:
\begin{align*}
  \minimize_{x: V\to [k]}\sum_{u\in V}c(u, x(u)) + \smashoperator{\sum_{(u,v) \in E}}w(u,v)\indicator[x(u) \ne x(v)].
\end{align*}
Here we are optimizing over \emph{labelings} $x: V\to [k]$ where $ [k] = \{1,2,\dots,k\}$. 
The objective is comprised of ``node costs'' $c: V\times [k] \to \mathbb{R}$, and ``edge weights'' $w: E \to \mathbb{R}_{> 0}$; a labeling $x$ pays the cost $c(u,i)$ when it labels node $u$ with label $i$ and pays $w(u,v)$ on edge $(u,v)$ when it labels $u$ and $v$ differently.
This problem is NP-hard for variable $k \ge 3$ \citep{KleinbergTardos02} even when the graph $G$ is planar \citep{dahlhaus1992complexity}.
However, there are several efficient and empirically successful approximation algorithms for the MAP inference problem---such as TRW \citep{wainwright2005map} and MPLP \citep{globerson2008fixing}---that are related in some way to the \emph{local LP relaxation}, which is also sometimes called the {\em pairwise LP} \cite{wainwright2008graphical, chekuri2001approximation}.  
This LP relaxation returns an approximate MAP solution for most problem instances. However, when the parameters of these models are learned so as to enable good structured prediction, often the LP relaxation exactly or almost exactly recovers the MAP solution  \citep{meshi2019train}.
The connection between the LP relaxation and commonly used approximate MAP inference algorithms then leads to the following compelling question, which is of great practical relevance for understanding the ``tightness'' of the LP solution (informally, how close the LP solution is to the MAP solution). 

{\em Can we explain the exceptional performance of the local LP relaxation in recovering the MAP solution in practice?}

\begin{figure*}[t]
\centering
\includegraphics[width=0.8\textwidth]{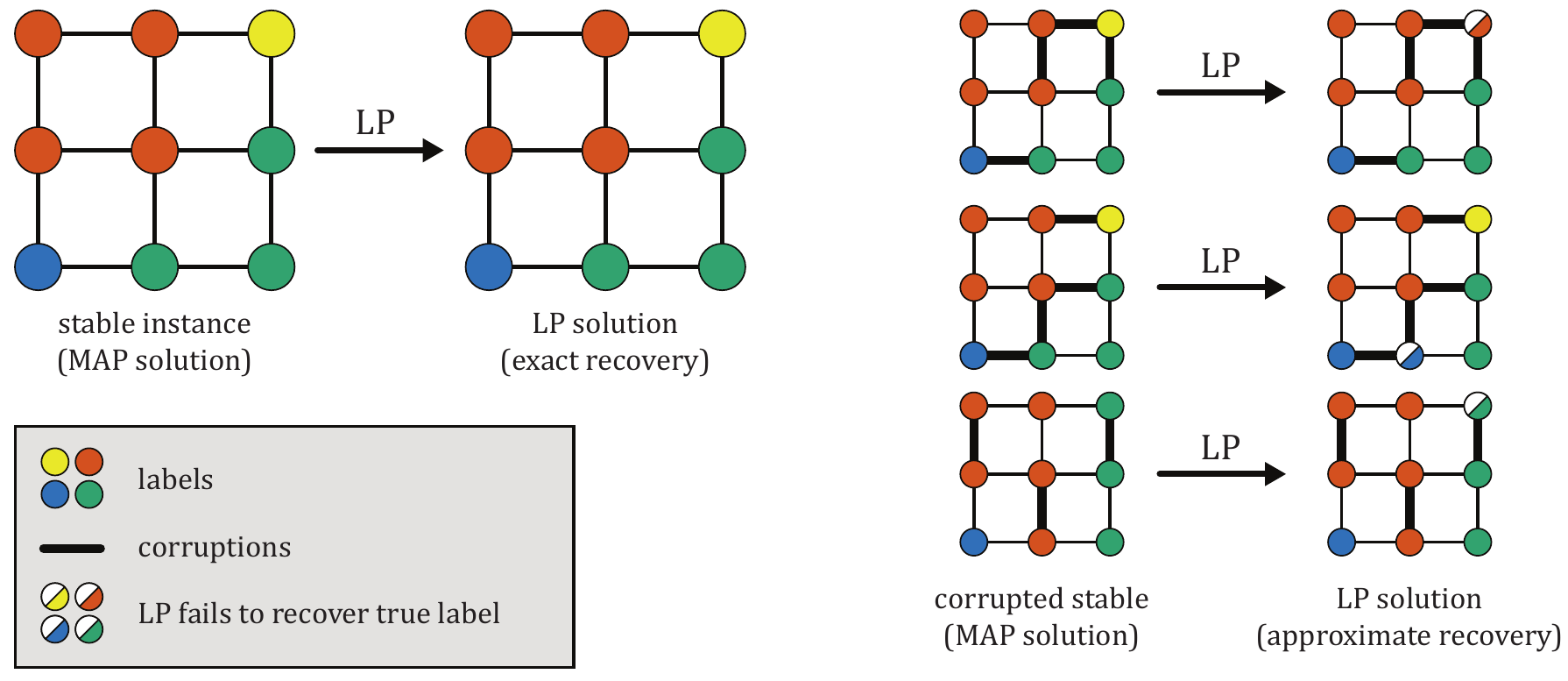}
\caption{Left: prior work \citep{LanSonVij18} showed that a \emph{stable instance} can be exactly solved efficiently. Colors indicate the label of each vertex in the MAP solution $x^*$. On stable instances, solving the LP relaxation (represented by the arrow) recovers the MAP solution. However, real-world instances are \emph{not} suitably stable for this result to apply in practice \citep{LanSonVij19}. Right: in this work, we show that solving the LP relaxation on a (slightly) \emph{corrupted} stable instance (corruptions shown as bold edges) \emph{approximately} recovers the original MAP solution. This is true even if the corruption changes the MAP solution (as in the bottom example). In other words, we prove that ``easy'' instances are still approximately easy even after small perturbations.}
\label{fig:main-idea}
\end{figure*}

Several works have studied different conditions that imply the local relaxation or related relaxations are tight \citep[e.g.,][]{kolmogorov2005optimality, wainwright2008graphical, thapper2012power, weller2016tightness,rowland2017conditions}.
Recent work on tightness of the local relaxation has focused on a class of several related conditions known as \emph{perturbation stability}.
Intuitively, an instance is perturbation stable if the solution $x^*$ to the MAP inference problem is unique, and moreover, $x^*$ is the unique solution even when the edge weights $w$ are multiplicatively perturbed by a certain adversarial amount \cite{BLstable}.
This structural assumption about the instance $(G,c,w)$ captures the intuition that, on ``real-world'' instances, the ground-truth solution is stable and does not change much when the weights are slightly perturbed.

For constants $\beta, \gamma \geq 1$, we say that $w'$ is a $(\beta,\gamma)$-perturbation of the weights $w$ if $\frac{1}{\beta}\cdot w(u,v) \le w'(u,v) \le \gamma \cdot w(u,v)$ for all $(u,v)\in E$. 
Suppose $x^*$ is the unique MAP solution to the instance $(G,c,w)$. Then, we say $(G,c,w)$ is a $(\beta,\gamma)$-stable instance if $x^*$ is also the unique MAP solution to every instance $(G,c,w')$ where $w'$ is a $(\beta,\gamma)$-perturbation of $w$.
 \citet{LanSonVij18} showed that when $(G,c,w)$ is $(2,1)$-stable, the solution to the local LP relaxation is {\em persistent} i.e., the LP solution exactly recovers the MAP solution $x^*$.

While theoretically interesting, $(2,1)$-stability is a strict condition that is unlikely to be satisfied in practice: the solution $x^*$ is not allowed to change \emph{at all} when the weights are perturbed. No real-world instances have yet been shown to be $(2,1)$-stable \citep{LanSonVij19}. Moreover, the LP relaxation is also not persistent on most of those instances. However, the solution of the local LP relaxation is still {\em nearly persistent} i.e., the LP solution is very close to the MAP solution $x^*$ (see Definition~\ref{def:hammingerror} for a formal definition). 
Those examples made it clear that theory must go beyond perturbation stability to explain this phenomenon of near-persistence that is prevalent in practice~\citep[see e.g.,][]{Sontagthesis, shekhovtsov2017maximum}.

{\em Why is the LP relaxation nearly persistent on MAP inference instances in practice?}

There are several theoretical frameworks to explain exactness or tightness of LP relaxations, such as total unimodularity, submodularity \citep{kolmogorov2005optimality}, and perturbation stability \citep{LanSonVij18, LanSonVij19}, as well as structural assumptions about the graph $G$ \citep{wainwright2008graphical}, or combined assumptions about $G$ and the form of the objective function \citep{weller2016tightness, rowland2017conditions}. However, these frameworks can not be used to prove near-persistence.

% \citet{LanSonVij19} proved that the linear program \emph{partially} recovers $x^*$ on large sub-portions (blocks) that are stable, and demonstrated that such large stable blocks are indeed present in many examples from computer vision. However, the required {\em block stability} condition in turn depends on certain quantities related to the LP dual. This is unsatisfactory since this does not explain when and why such instances are likely to arise in practice. Developing new techniques and frameworks to reason about near-persistence presents a technical challenge that may also be of interest in other settings. \fxnote{move to related work}

Figure \ref{fig:main-idea} (informally) shows our main result. The left side depicts the previous result of \citet{LanSonVij18}: if the instance is $(2,1)$-stable (a fairly strong structural assumption), the LP relaxation exactly recovers the full solution $x^*$. This result is limited because real-world instances have been shown to not satisfy $(2,1)$-stability \citep{LanSonVij19}. The right side shows our main result: if the instance is a \emph{slightly corrupted} $(2,1)$-stable instance, the LP relaxation still \emph{approximately} recovers the solution $x^*$ to the stable instance.

Intuitively, we may expect a real-world instance to be ``close'' to a stable instance (i.e., to be a ``corrupted stable'' instance, as in Figure \ref{fig:main-idea}) even if the instance itself is not stable. We design an algorithm to check whether this is the case. We find that on several real examples, sparse and small-norm changes to the instance make it appropriately stable for our theorems to apply. In other words, we certify that these real instances are close to stable instances. For these instances, our theoretical results explain why the LP relaxation approximately recovers the MAP solution.

More formally, we assume that there is some \emph{latent} stable instance $\stabins$, and that the observed instance $\obsins$ is a noisy version of $\stabins$ that is close to it.
Let $\hx$ be the solution to the local LP on the observed instance $\obsins$, and let $\bx$ be the (unknown) MAP solution on the unseen stable instance $(G,\bar{c},\bar{w})$. We prove that under certain conditions, the LP solution $\hx$ is {\em nearly persistent} i.e., the Hamming error $\norm{\hx - \bx}_1$ is small (see Definition~\ref{def:hammingerror}). In other words, the local LP solution to the observed instance approximately recovers the latent integral solution $\bx$.

We complement this by studying a natural generative model that generates noisy stable instances which, with high probability, satisfy the above conditions for {\em near persistency}. In other words, the observed instance $\obsins$ is obtained by random perturbations to the latent stable instance $\stabins$, and the LP relaxation approximately recovers the MAP solution to the latent instance with high probability.

Our theoretical results imply that the local LP approximately recovers the MAP solution when the observed instance is close to a stable instance. Our empirical results suggest that real-world instances are very close to stable instances. These results together suggest a new explanation for the near-persistency of the solution of the local LP relaxation for MAP inference in practice. To prove these results and derive our algorithm for finding a ``close-by'' stable instance, we make several novel technical contributions, which we outline below.

\paragraph{Technical contributions.}
\begin{itemize}
    \item In Section \ref{sec:expansion-stability}, we generalize the $(2,1)$-stability result of \citet{LanSonVij18} to work under a much weaker assumption, which we call $(2,1)$-expansion stability. That is, we prove the local LP is tight on $(2,1)$-expansion stable instances. 
    Additionally, given the instance's MAP solution, $(2,1)$-expansion stability is efficiently checkable. To the best of our knowledge, most other perturbation stability assumptions are not known to satisfy this desirable property. This generalization is crucial for the efficiency of our algorithm for finding stable instances that are close to a given observed instance.
    \item In Section \ref{sec:stable-curvature}, we give a simple extension of $(2,1)$-expansion stability called $(2,1,\psi)$-expansion stability. We prove it implies a ``curvature'' result around the MAP solution $\bx$. On instances that satisfy this condition, if a labeling $\hat{x}$ is close in objective value to $\bx$, it must also be close in the solution space. This result lets us translate between objective gap and Hamming distance.
    \item In Section \ref{sec:random-model}, we study a natural generative model where the observed instance is generated from an arbitrary {\em latent stable} instance by random (sub-Gaussian) perturbations to the costs and weights. We prove that, with high probability, every feasible LP solution takes close objective values on the latent and observed instances. The proof uses a rounding algorithm for metric labeling in a novel way to obtain stronger guarantees. When combined with our other results, this proves that when the latent instance is $(2,1,\psi)$-expansion stable, the LP solution is nearly persistent on the observed instance with high probability. These results suggest a theoretical explanation for the phenomenon of near-persistence of the LP solution in practice. 
    \item We design an efficient algorithm for finding $(2,1,\psi)$-expansion stable instances that are ``close'' to a given instance $\obsins$ in Section \ref{sec:algorithm}. To the best of our knowledge, this is the first algorithm for finding close-by stable instances, and is also an efficient algorithm for checking $(2,1,\psi)$-expansion stability. This algorithm allows us to check whether real-world instances can plausibly be considered ``corrupted stable'' instances as shown in Figure \ref{fig:main-idea}.
    \item We run our algorithm on several real-world instances of MAP inference in Section \ref{sec:experiments}, and find that the observed instances $\obsins$ often admit close-by $(2,1,\psi)$-stable instances $\stabins$. 
    Moreover, we find that the local LP solution $\hat{x}$ typically has very close objective to $\bx$ in $\stabins$.
    Our curvature result for $(2,1,\psi)$-stable instances thus gives an explanation for the tightness of the local LP relaxation on $\obsins$.
\end{itemize}

\section{Related work}
\label{sec:related}
\paragraph{Perturbation stability.}

Several works have given recovery guarantees for the local LP relaxation on perturbation stable instances of uniform metric labeling \citep{LanSonVij18, LanSonVij19} and for similar problems \citep{makarychev2014bilu, AngMakMak17}.

\citet{LanSonVij19} give partial recovery guarantees for the local LP when parts (\textit{blocks}) of the observed instance satisfy a stability-like condition, and they showed that practical instances have blocks that satisfy their condition. However, the required {\em block stability} condition in turn depends on certain quantities related to the LP dual. This is unsatisfactory since this does not explain when and why such instances are likely to arise in practice.
For a more extensive treatment of the subject, we refer the reader to the ``Perturbation Resilience'' chapter from \citet{roughgarden_2021}.

\paragraph{Easy instances corrupted with noise.}
Our random noise model is similar to several planted average-case models like stochastic block models (SBMs) considered in the context of problems like community detection, correlation clustering and partitioning \citep[see e.g.,][]{Mcsherry,Abbesurvey, GRSY15}. Instances generated from these models can also be seen as the result of random noise injected into an instance with a nice block structure that is easy to solve. 
Several works give exact recovery and approximate recovery guarantees for semidefinite programming (SDP) relaxations for such models in different parameter regimes~\cite{Abbesurvey,guedon2016community}. 
In our model however, we start with an {\em arbitrary} stable instance as opposed to an instance with a block structure (which is trivial to solve). 
Moreover, we are unaware of such analysis in the context of linear programs.
Please see Section~\ref{sec:random-model} for a more detailed comparison. 
To the best of our knowledge, we are the first to study instances generated from random perturbations to stable instances.

\paragraph{Partial optimality algorithms.} Several works have developed fast algorithms for identifying parts of the MAP assignment. These algorithms output an approximate solution $\hat{x}$ and a set of vertices where $\hat{x}$ provably agrees with the MAP solution $x^*$ \citep[e.g.,][]{kovtun2003partial, shekhovtsov2013exact,swoboda2016partial,shekhovtsov2017maximum}. Like these works, our results also prove that an approximate solution $\hat{x}$ has small error $|\hat{x} - x^*|$. However, these previous approaches are more concerned with designing fast algorithms for finding such $\hat{x}$. In contrast, we focus on giving structural conditions that explain why a \emph{particular} $\hat{x}$ (the solution to the local LP relaxation) often approximately recovers $x^*$. Our algorithm in Section \ref{sec:algorithm} is thus not meant as an efficient method for certifying that $|\hat{x} - x^*|$ is small, but rather as a method for checking whether our structural condition (that the observed instance is close to a stable instance) is satisfied in practice.

\section{Preliminaries}
\label{sec:prelim}
In this section we introduce our notation, define the local LP relaxation for MAP inference, and give more details on perturbation stability.
As in the previous section, the MAP inference problem in the ferromagnetic Potts model on the instance $(G,c,w)$ can be written in \emph{energy minimization} form as:
\begin{equation}
\label{eqn:map-problem}
  \minimize_{x: V\to [k]}\sum_{u\in V}c(u, x(u)) + \smashoperator{\sum_{(u,v) \in E}}w(u,v)\indicator[x(u) \ne x(v)].
\end{equation}
Here $x$ is an \emph{assignment} (or labeling) of vertices to labels i.e. $x: V\to \{1,2,\dots,k\}$. We can identify each labeling $x$ with a point $(x_u : u\in V; x_{uv} : (u,v) \in E)$, where each $x_u \in \{0,1\}^k$ and each $x_{uv} \in \{0,1\}^{k\times k}$. 

In this work, we consider all node costs $c(u,i) \in \R$ and all edge weights $w(u,v) > 0$. We note that this is equivalent to the formulation where all node costs and edge weights are non-negative~\cite{KleinbergTardos02}. See Appendix~\ref{sec:prelim_details} for a proof of this equivalence. 

We encode the node costs and the edge weights in a vector $\theta \in \R^{nk + mk^2}$ where $n = \abs{V} \text{ and } m = \abs{E}$ s.t. $\theta(u,i) = c(u,i), \theta(u,v,i,j) = w(u,v)\indicator[i \ne j]$. Then the objective can be written as $\dot{\theta}{x}$. We set $x_u(i) = 1$ when $x(u) = i$, and 0 otherwise. Similarly, we set $x_{uv}(i,j) = 1$ when $x(u) = i$ and $x(v) = j$, and 0 otherwise. Where convenient, we use $x$ to refer to this point rather than the labeling $x: V \to [k]$.
We can then rewrite \eqref{eqn:map-problem} as:
\begin{alignat*}{2}
  \mindot_{x}\sum_{u\in V}&\sum_{i=1}^k c(u,i)x_u(i) + \smashoperator{\sum_{(u,v) \in E}}\ \ &&w(u,v)\sum_{i\ne j}x_{uv}(i,j)\\
  \text{subject to:}& \sum_{i=1}^k x_u(i) = 1 &&\forall\ u\in V\\
                    & \sum_{i=1}^k x_{uv}(i,j) = x_v(j)&& \forall\ (u,v)\in E,\ j\in [k]\\
                   & \sum_{j=1}^k x_{uv}(i,j) = x_u(i)&& \forall\ (u,v)\in E,\ i\in [k]\\
                   & x_{uv}(i,j) \in \{0,1\}&&\forall\ (u,v),\ (i,j)\\
                   & x_{u}(i) \in \{0,1\}&& \forall\ u,\ i.
\end{alignat*}
This is equivalent to \eqref{eqn:map-problem}, and is an integer linear program (ILP). By relaxing the integrality constraints from $\{0,1\}$ to $[0,1]$, we obtain the \emph{local LP relaxation}:
\begin{alignat*}{2}
  \mindot_{x\in L(G)}\sum_{u\in V}\sum_{i=1}^k c(u,i)x_u(i) + \smashoperator{\sum_{(u,v) \in E}}w(u,v)\sum_{i\ne j}x_{uv}(i,j),
\end{alignat*}
where $L(G)$ is the polytope defined by the same constraints as above, with $x \in \{0,1\}$ replaced with $x \in [0,1]$. This is known as the \emph{local polytope} \cite{wainwright2008graphical}. The vertices of $L(G)$ are either \emph{integral}, meaning all $x_u$ and $x_{uv}$ take values in $\{0,1\}$, or \emph{fractional}, when some variables take values in $(0,1)$. 
Integral vertices of this polytope correspond to labelings ${x: V\to [k]}$, so if the LP solution is obtained at an integral vertex, then it is also a MAP assignment.

If the solution $x^*$ of this relaxation on an instance $(G,c,w)$ is obtained at an integral vertex, we say the LP is \emph{tight} on the instance, because the LP has exactly recovered a MAP assignment. If the LP is not tight, there may still be some vertices $u$ where $x^*_u$ takes integral values. In this case, if $x^*_u(i) = 1$ and $\bar{x}(u) = i$, i.e., the LP solution agrees with the MAP assignment $\bar{x}$ at vertex $u$, the LP is said to be \emph{persistent} at $u$. $x^*_u(i) \in \{0,1\}$ does not imply the LP is persistent at $u$, in general. The LP solution $x^*$ is said to be persistent if it agrees with $\bar{x}$ at every vertex $u \in V$. 

\vspace{2pt}
\noindent {\em Recovery error:} In practice, the local LP relaxation is often not tight, but is \emph{nearly persistent}.
We will measure the recovery error of our LP solution in terms of the ``Hamming error'' between the LP solution and the MAP assignment. 

\begin{definition}[Recovery error]\label{def:hammingerror}
Given an instance $(G,c,w)$ of \eqref{eqn:map-problem}, let $\bar{x}$ be a MAP assignment, and let $x^*$ be a solution to the local LP relaxation. The recovery error  is given by (with some abuse of notation)
\begin{align*}
\frac{1}{2}\|x^* - \bar{x}\|_1 &:= \frac{1}{2}\|x^*_V - \bar{x}_V\|_1 \\
&=\frac{1}{2}\sum_{ u \in V} \sum_{i \in [k]} \big| x^*_u(i) - \indicator[\bar{x}(u) = i ] \big|.
\end{align*}
\end{definition}
$x_V \in \R^{nk}$ denotes the portion of $x$ restricted to the vertex set $V$. If $x^*$ is integral, the recovery error measures the number of vertices where $x^*$ disagrees with $\bar{x}$. When the recovery error of $x^*$ is $0$, the solution $x^*$ is {\em persistent}. 
We will say that the LP solution $x^*$ is {\em nearly persistent} when the recovery error of solution $x^*$ is a small fraction of $n$.

In our analysis, we will consider the following subset $L^*(G)$ of $L(G)$ which is easier to work with, and which contains all points we are interested in.

\begin{definition}[$L^*(G)$]
\label{def:L*}
We define $L^*(G) \subseteq L(G)$ to be the set of points $x \in L(G)$ which further satisfy the constraint that $x_{uv}(i,i) = \min(x_u(i), x_v(i))$ for all $(u,v) \in E$ and $i \in [k]$.
\end{definition}

\begin{restatable}[]{claim}{Lstarclaim}\label{claim:Lstar}
For a given graph $G$, every solution $x \in L(G)$ that minimizes $\dot{\theta}{x}$ for some valid objective vector $\theta=(c,w)$ also belongs to $L^*(G)$. Further, all integer solutions in $L(G)$ also belong to $L^*(G)$. 
\end{restatable}
We prove this claim in Appendix~\ref{sec:prelim_details}.

Our new stability result relies on the set of \emph{expansions} of a labeling $x$.
\begin{definition}[Expansion]
\label{def:expansion}
Let $x: V\to [k]$ be a labeling of $V$. For any label $\alpha \in [k]$, we say that $x'$ is an $\alpha$-expansion of $x$ if $x' \ne x$ and the following hold for all $u\in V$:
\begin{align*}
x(u) = \alpha &\implies x'(u) = \alpha,\\
x'(u) \ne \alpha &\implies x'(u) = x(u).
\end{align*}
That is, $x'$ may only expand the set of points labeled $\alpha$, and cannot make other changes to $x$.
\end{definition}

\section{Expansion Stability}\label{sec:expansion-stability}
\begin{figure}[tb]
  \centering
\begin{subfigure}{.5\linewidth}
  \centering
  \scalebox{1.0}{  
  \tikzstyle{vertex}=[circle, draw=black, very thick, minimum size=5mm]
  \tikzstyle{edge} = [draw=black, line width=1]
  \tikzstyle{weight} = [font=\normalsize]
  \begin{tikzpicture}[scale=2,auto,swap]
    \foreach \pos /\name in {{(0,0)}/u,{(1,0)}/w,{(0.5,0.75)}/v}
    \node[vertex](\name) at \pos{$\name$};
    \foreach \source /\dest /\weight in {u/w/1+\epsilon}
    \path[edge] (\source) -- node[weight] {$\weight$} (\dest);
    \foreach \source /\dest /\weight/\pos in {u/v/1+\epsilon/{above left}, v/w/1+\epsilon/{above right}}
    \path[edge] (\source) -- node[weight, \pos] {$\weight$} (\dest);
  \end{tikzpicture}
  }
\end{subfigure}%
\begin{subfigure}{.5\linewidth}
  \centering
  \scalebox{1.0}{
\begin{tabular}{|l|ccc|}
\hline
\multicolumn{1}{|c|}{\textbf{Node}} & \multicolumn{3}{|c|}{\textbf{Costs}} \\
\hline
u & .5 & $\infty$        & $\infty$      \\
\hline
v & 1 & 0 & $\infty$\\
\hline
w & 1 & $\infty$ & 0\\
\hline
\end{tabular}
}
\end{subfigure}
  \caption{$(2,1)$-expansion stable instance that is not $(2,1)$-stable. In the original instance (shown left), the optimal solution labels each vertex with label 1, for an objective of $2.5$. This instance is not $(2,1)$-stable: consider the $(2,1)$-perturbation that multiples all edge weights by $1/2$. In this perturbed instance, the original solution still has objective 2.5, and the new optimal solution labels $(u,v,w) \rightarrow (1,2,3)$. This has a node cost of 0.5 and an edge cost of $(3+3\epsilon)/2$, for a total of $2+3\epsilon/2 < 2.5$. Since the original solution is not optimal in the perturbed instance, this instance is not $(2,1)$-perturbation stable. However, note that the only expansions of the original solution (which had all label 1) that have non-infinite objective are $(u,v,w) \rightarrow (1,2,1)$ and $(u,v,w) \rightarrow (1,1,3)$. These each have objective $2.5 + \epsilon$, which is strictly greater than the perturbed objective of the original solution. In fact, checking this single perturbation, known as the \emph{adversarial perturbation} is enough to verify expansion stability: this instance is $(2,1)$-expansion stable. We include the full details in Appendix \ref{sec:expansion-stability_details}.}
\label{fig:counter1}
\end{figure}
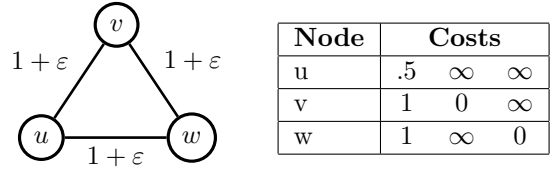

In this section, we generalize the stability result of \citet{LanSonVij18} to a much broader class of instances.
This generalization allows us to efficiently check whether a real-world instance could plausibly have the structure shown in Figure \ref{fig:main-idea} (that is, whether the instance is close to a suitably stable instance).

Consider a fixed instance $(G,c,w)$ with a unique MAP solution $\bar{x}$.
Theorem 1 of \citet{LanSonVij18} requires that for all $\theta' \in \{(c,w')\ |\ w' \in \{(2,1)\text{-perturbations of } w\}\}$, $\langle \theta', x \rangle > \langle \theta', \bx \rangle$ for \emph{all} labelings $x\ne \bx$.
That is, that result requires $\bx$ to be the unique optimal solution in any $(2,1)$-perturbation of the instance.
By contrast, our result only requires $\bx$ to have better perturbed objective than the set of \emph{expansions} of $\bx$ (c.f. Definition \ref{def:expansion}).
\begin{definition}[(2,1)-expansion stability]
\label{def:expansion-stability}
    Let $\bx$ be the unique MAP solution for $(G,c,w)$, and let $\calE_{\bx}$ be the set of expansions of $\bx$ (see Definition \ref{def:expansion}). 
    Let \[
    \Theta = \{(c,w')\ |\  w' \in \{(2,1)\text{-perturbations of } w\}\}
    \]
    be the set of all objective vectors within a $(2,1)$-perturbation of $\theta = (c,w)$. We say the instance $(G,c,w)$ is $(2,1)$-expansion stable if the following holds for all $\theta' \in \Theta$ and all $x\in \calE_{\bx}$:
    \[
        \dot{\theta'}{x} > \dot{\theta'}{\bx}.
    \]
    That is, $\bx$ is better than all of its expansions $x \ne \bx$ in every $(2,1)$-perturbation of the instance.
\end{definition}

\begin{restatable}[Local LP is tight on $(2,1)$-expansion stable instances]{theorem}{lptes}\label{thm:lptes}
Let $\bx$ and $\hx$ be the MAP and local LP solutions to a $(2,1)$-expansion stable instance $(G,c,w)$, respectively. Then $\bx = \hx$ i.e. the local LP is tight on $(G,c,w)$.
\end{restatable}
We defer the proof of this theorem to Appendix~\ref{sec:expansion-stability_details} as it is similar to the proof of Theorem 1 from \citet{LanSonVij18}. 
The $(2,1)$-expansion stability assumption is much weaker than $(2,1)$-stability because the former only compares $\bx$ to its expansions, whereas the latter compares $\bx$ to \emph{all} labelings.
While the rest of our results can also be adapted to the $(2,1)$-stability definition, this relaxed assumption gives better empirical results. Figure \ref{fig:counter1} shows an example of a $(2,1)$-expansion stable instance that is not $(2,1)$-stable. This shows that our new stability condition is less restrictive.

\section{Curvature around MAP solution and near persistence of the LP solution}\label{sec:stable-curvature}
In this section, we show that a condition related to $(2,1)$-expansion stability, called $(2,1,\psi)$-expansion stability, implies a ``curvature'' result for the objective function around the MAP solution $\bx$. On instances satisfying this condition, any point $\hx\in L(G)$ with objective close to $\bx$ also has small $\norm{\hx-\bx}_1$, so $\hx$ and $\bx$ are close in solution space. In other words, if the LP solution $\hat{x}$ to a ``corrupted'' $(2,1,\psi)$-expansion stable instance is near-optimal in the original $(2,1,\psi)$-expansion stable instance (whose solution is $\bx$), then the result in this section implies $\norm{\hx-\bx}_1$ is small. This immediately gives a version of the result in the right panel of Figure~\ref{fig:main-idea}: suppose we define an instance to be \emph{close} to a $(2,1,\psi)$-expansion stable $(G,\bar{c},\bar{w})$ if its LP solution $\hat{x}$ is approximately optimal in $(G,\bar{c},\bar{w})$. Then the curvature result implies that the LP approximately recovers the stable instance's MAP solution $\bx$ for all close instances. In Section \ref{sec:random-model}, we give a generative model where the generated instances are ``close'' according to this definition with high probability.

The $(2,1,\psi)$-expansion stability condition, for $\psi > 0$, says that the instance is $(2,1)$-expansion stable even if we allow all node costs $c(u,i)$ to be additively perturbed by up to $\psi$. This extra additive stability will allow us to prove the curvature result.
This is related to the use of additive stability in \citet{LanSonVij19} to give persistency guarantees.
\begin{definition}[$(2,1,\psi)$-expansion stable] For $\psi > 0$, we say an instance $(G,c,w)$ is $(2,1,\psi)$-expansion stable if $(G,c',w)$ is $(2,1)$-expansion stable for all $c'$ with $c \le c' \le c+ \psi \cdot \mathbf{1}$ where $\mathbf{1}$ is the all-ones vector.
\end{definition}

The following theorem shows low recovery error i.e., near persistence of the LP solution on $(2,1,\psi)$ expansion stable instances in terms of the gap in objective value.  
\begin{restatable}{theorem}{curvature}\label{thm:curvature}
Let $(G,c,w)$ be a $(2,1,\psi)$-expansion stable instance with MAP solution $\bar{x}$. 
Let $\theta = (c,w)$. Then for any $x \in L^*(G)$, the recovery error (see Def.~\ref{def:hammingerror}) satisfies
\begin{equation}\label{eq:deviationLP}
    \frac{1}{2}\norm{x - \bx}_1:=\frac{1}{2}\norm{x_V - \bx_V}_1 \le \frac{1}{\psi}\abs{\dot{\theta}{x}-\dot{\theta}{\bx}}.
\end{equation}
\end{restatable}
\begin{proof}[Proof (sketch)]
For any $x \in L^*(G)$, we construct a feasible solution $\hat{x}$ which is a strict convex combination of $x$ and $\bar{x}$ that is very close to $\bar{x}$. Then, we apply a rounding algorithm to $\hx$ to get a random integer solution $h$. Let $\hat\theta$ represent the worst $(2,1)$-perturbation for $\bar{x}$. This is the instance where all the edges not cut by $\bar{x}$ have their weights multiplied by $1/2$. We define the objective difference using $\hat\theta$ as $A_h = \dot{\hat\theta}{h} - \dot{\hat\theta}{\bar{x}}$. First we show an upper bound for $\E[A_h]$ using properties of the rounding algorithm. Then we show that for any solution $h$ in the support of our rounding algorithm, $A_h \geq \psi \cdot B_h$ where $B_h$ is the Hamming error of $h$ (when compared to $\bx$). 
On the other hand, one can also use the properties of the rounding algorithm to get a lower bound on $\E[B_h]$ in terms of the recovery error (i.e., Hamming error) of the LP solution. These bounds together imply the required upper bound on the recovery error of the LP solution. 
\end{proof}
We defer the complete proof and an alternate dual-based proof to Appendix~\ref{sec:stable-curvature_details}. 

Theorem \ref{thm:curvature} shows that on a $(2,1,\psi)$-expansion stable instance, small objective gap $\dot{\theta}{x} - \dot{\theta}{\bx}$ implies small distance $||x_V-\bx_V||_1$ in solution space. Although this holds for any $x \in L^*(G)$, we will be interested in $x$ that are LP solutions to an observed, corrupted version of the stable instance.

We now show that if the observed instance has a nearby stable instance, then the LP solution for the observed instance has small Hamming error. For any two instances $\hat{\theta} = (\hat{c},\hat{w})$ and $\bar{\theta} = (\bar{c}, \bar{w})$ on the same graph $G$, the metric between them $d(\hat{\theta},\bar{\theta}) \coloneqq \sup_{x \in L^*(G)} \abs{ \dot{\hat{\theta}}{x}  - \dot{\bar{\theta}}{x}}$.

\begin{restatable}[LP solution is good if there is a nearby stable instance]{corollary}{deviation}\label{cor:deviation}
Let $\hx^{MAP}$ and $\hx$ be the MAP and local LP solutions to an observed instance $\obsins$. Also, let $\bar{x}$ be the MAP solution for a latent $(2,1,\psi)$-expansion stable instance $\stabins$. If $\hat{\theta} = (\hat{c},\hat{w})$ and $\bar{\theta} = (\bar{c}, \bar{w})$, \[ \frac{1}{2}\norm{\hx_V - \hx^{MAP}_V}_1 \le \frac{2d(\hat{\theta}, \bar{\theta})}{\psi} + \frac{1}{2}\norm{\hx_V^{MAP} - \bx_V}_1. \]
\end{restatable}

We defer the proof of this corollary to Appendix~\ref{sec:stable-curvature_details}.

\section{Generative model for noisy stable instances}\label{sec:random-model}
In the previous section, we showed that if an instance $(G, \hc, \hw)$ is ``close'' to a $(2,1,\psi)$-expansion stable instance $(G,\bc, \bw)$ (i.e., the LP solution $\hat{x}$ to $(G, \hc, \hw)$ has good objective in  $(G,\bc, \bw)$), then $\norm{\hx - \bx}$ is small, where $\bx$ is the MAP solution to the stable instance. In this section, we give a natural generative model for $(G, \hc, \hw)$, based on randomly corrupting $(G,\bc,\bw)$, in which $\hat{x}$ has good objective in $(G,\bc,\bw)$ with high probability. Together with the curvature result from the previous section (Theorem \ref{thm:curvature}), this implies that the LP relaxation, run on the noisy instances $(G,\hc,\hw)$, approximately recovers $\bx$ with high probability.

We now describe our generative model for the problem instances, which starts with an arbitrary stable instance and perturbs it with random additive perturbations to the edge costs and node costs (of potentially varying magnitudes). 
The random perturbations reflect possible uncertainty in the edge costs and node costs of the Markov random field.
We will assume the random noise comes from any distribution that is sub-Gaussian\footnote{All of the results that follow can also be generalized to sub-exponential random variables; however for convenience, we restrict our attention to sub-Gaussians.}. However, there is a small technicality: the edge costs need to be positive (node costs can be negative). For this reason we will consider truncated sub-Gaussian random variables for the noise for the edge weights. We define sub-Gaussian and truncated sub-Gaussian random variables in Appendix~\ref{sec:random-model_details}.

\paragraph{Generative Model:}
We start with an instance $(G,\bar{c},\bar{w})$ that is $(2, 1, \psi)$-expansion stable, and perturb the edge costs and node costs independently. Given any instance $(G,\bar{c},\bar{w})$, an instance $(G,\hat{c},\hat{w})$ from the model is generated as follows:

\begin{enumerate}[noitemsep,nolistsep]
    \item For all node-label pairs $(u,i),\ \hat{c}(u,i) = \bar{c}(u,i) + \tc(u,i)$, where $\tc(u,i)$ is sub-Gaussian with mean $0$ and parameter $\sigma_{u,i}$.
    \item For all edges $(u,v),\ \hat{w}(u,v) = \bar{w}(u,v) + \tw(u,v)$, where $\tw(u,v)$ is an independent r.v. that is $(-w(u,v),\gamma_{u,v})$-truncated sub-Gaussian with mean $0$.
    \item $(G, \hat{c}, \hat{w})$ is the observed instance.
\end{enumerate}

By the definition of our model, the edge weights $\hat{w}(u,v) \ge 0$ for all $(u,v) \in E$. The parameters of the model are the unperturbed instance $(G,\bar{c},\bar{w})$, and the noise parameters $\set{\gamma_{u,v},\sigma_{u,i}}_{u,v \in V, i \in [k]}$. 
On the one hand, the above model captures a natural average-case model for the problem. For a fixed ground-truth solution $x^*: V \to [k]$, consider the stable instance $(H,c,w)$ where $w^*_{uv}=2$ for all $u,v$ in the same cluster (i.e., $x^*(u) = x^*(v)$) and $w^*_{uv}=1$ otherwise; and with $c^*(u,i)=1$ if $x^*(u)=i$, and $c^*(u,i)=1+\psi$ otherwise. The above noisy stable model with stable instance $(H,c,w)$ generates instances that are reminiscent of (stochastic) block models, with additional node costs. On the other hand, the above model is much more general, since we can start with {\em any} stable instance $(G,c,w)$.   

With high probability over the random corruptions of our stable instance, the local LP on the corrupted instance approximately recovers the MAP solution $\bar{x}$ of the stable instance. The key step in the proof of this theorem is showing that, with high probability, the observed instance is close to the latent stable instance in the metric we defined earlier.

\begin{restatable}[$d(\hat{\theta},\bar{\theta})$ is small w.h.p. ]{lemma}{dswhp}\label{lem:dswhp}
There exists a universal constant $c < 1$ such that for any instance in the above model, with probability at least $1-o(1)$, 
\begin{equation*}
    \smashoperator{\sup_{x \in L^*(G)}} \abs{\dot{\hat{\theta}}{x}  - \dot{\bar{\theta}}{x}}\le c \sqrt{nk} \sqrt{\sum_{u,i}\sigma_{u,i}^2 + \frac{k^2}{4}\sum_{uv}\gamma_{u,v}^2}
\end{equation*}
\end{restatable}

\begin{proof}[Proof (sketch)]
For any \emph{fixed} $x \in L^*(G)$, we can show that $\abs{\dot{\hat{\theta}}{x}  - \dot{\bar{\theta}}{x}}$ is small w.h.p. using a standard large deviations bound for sums of sub-Gaussian random variables. The main technical challenge is in showing that the supremum over all feasible solutions is small w.h.p. The standard approach is to perform a union bound over an $\epsilon$-net of feasible LP solutions in $L^*$. However, this gives a loose bound. Instead, we upper bound the supremum by using a rounding algorithm for LP solutions in $L^*(G)$, and union bound only over the discrete solutions output by the rounding algorithm. This gives significant improvements over the standard approach; for example, in a $d$-regular graph with equal variance parameter $\gamma_{uv}$, this saves a factor of $\sqrt{d}$ apart from logarithmic factors in $n$.
\end{proof}

We defer the details to Appendix~\ref{sec:random-model_details}. 
The above proof technique that uses a rounding algorithm to provide a deviation bound for a continuous relaxation is similar to the analysis of SDP relaxations for average-case problems~\citep[see e.g.,][]{MMVfas,guedon2016community}. 
The above lemma, when combined with Theorem~\ref{thm:curvature} gives the following guarantee. 

\begin{restatable}[LP solution is nearly persistent]{theorem}{apmap}\label{thm:apmap}
Let $\hx$ be the local LP solution to the observed instance $\obsins$ and $\bar{x}$ be the MAP solution to the latent $(2,1,\psi)$-expansion stable instance $\stabins$. With high probability over the random noise, \[ \frac{1}{2}\norm{\hx_V - \bar{x}_V}_1 \le \frac{2}{\psi} \cdot c\sqrt{nk} \cdot \sqrt{\sum_{u,i}\sigma_{u,i}^2 + k^2\sum_{uv}\gamma_{u,v}^2}\]

\end{restatable}
\begin{proof}
We know that for any feasible solution $x \in L(G), \dot{\bar{\theta}}{x} \ge \dot{\bar{\theta}}{\bx}$. Therefore, $\dot{\bar{\theta}}{\hx} \ge \dot{\bar{\theta}}{\bx}$. Remember that we defined $d(\hat{\theta}, \bar{\theta})$ as $\sup_{x \in L^*(G)} \abs{ \dot{\hat{\theta}}{x}  - \dot{\bar{\theta}}{x}}$. Since $\hx$ and $\bx$ are both points in $L^*(G)$,
\begin{align*}
\dot{\bar{\theta}}{\hx} \leq \dot{\hat{\theta}}{\hx} + d(\hat{\theta},\bar{\theta}) \leq \dot{\hat{\theta}}{\bx} &+ d(\hat{\theta},\bar{\theta}) \\
&\leq \dot{\bar{\theta}}{\bx} + 2d(\hat{\theta},\bar{\theta})
\end{align*}
The first and third inequalities follow from the definition of $d(\hat{\theta},\bar{\theta})$. The second inequality follows from the fact that $\hat{x}$ is the minimizer of $\dot{\bar{\theta}}{x}$ over all $x \in L(G)$.
Therefore,$0 \le \dot{\bar{\theta}}{\hx} - \dot{\bar{\theta}}{\bx} \leq 2d(\hat{\theta},\bar{\theta})$. Using this in Theorem~\ref{thm:curvature}, we get $\frac{1}{2}\norm{\hx - \bar{x}}_1 \le \frac{2d(\hat{\theta},\bar{\theta})}{\psi}$. Lemma~\ref{lem:dswhp} then gives an upper bound on $d(\hat{\theta},\bar{\theta})$ that holds w.h.p.
\end{proof}

For a $d$-regular graph in the uniform setting, we get the following useful corollary:
\begin{restatable}[MAP solution recovery for regular graphs ]{corollary}{mapreg}\label{cor:mapreg}
Suppose we have a $d$-regular graph $G$ with $\gamma_{u,v}^2 = \gamma^2$ for all edges $(u,v)$, and $\sigma_{u,i}^2 = \sigma^2$ for all vertex-label pairs $(u,i)$. Also, suppose only a fraction $\rho$ of the vertices and $\eta$ of the edges are subject to the noise. With high probability over the random noise, \[ \frac{\norm{\hx_V - \bar{x_V}}_1}{2n} \le \frac{2ck\sqrt{\rho \sigma^2 + \frac{\eta dk}{8}\gamma^2}}{\psi}\]
\end{restatable}
Note that when $\hat{x}$ is an integer solution, the left-hand-side is the fraction of vertices misclassified by $\hat{x}$.

\section{Finding nearby stable instances}\label{sec:algorithm}
In this section, we describe an algorithm for efficiently finding $(2,1,\psi)$-expansion stable instances that are ``close'' to an observed instance $\obsins$. This algorithm allows us to check whether we can plausibly model real-world instances as ``corrupted'' versions of a $(2,1,\psi)$-expansion stable instance.

In addition to an observed instance $\obsins$, the algorithm takes as input a ``target'' labeling $x^t$. For example, $x^t$ could be a MAP solution of the observed instance. Surprisingly, once given a target solution, this algorithm is efficient.

We want to search over costs $c$ and weights $w$. 
The broad goal to solve the following optimization problem:
\begin{align}
\label{eqn:orig-searchprob}
\minimize_{c\ge 0,w\ge 0} \qquad & f(c,w)\\
\text{subject to} \qquad & (G,c,w) \text{ is } (2,1,\psi)\text{-expansion stable}\nonumber\\
\qquad & \text{with MAP solution } x^t\nonumber,
\end{align}
where $f(c,w)$ is any convex function of $c$ and $w$. 
In particular, we will use $f_1(c,w) = ||(c,w)-(\hc,\hw)||_1$
and $f_2(c,w) = \frac{1}{2}||(c,w) - (\hc,\hw)||_2^2$ for minimizing the L1 and L2 distances between to the observed instance.
The output of this optimization problem will give the closest objective vector $(\bc,\bw)$ for which the instance $\stabins$ is $(2,1,\psi)$-expansion stable.
If the optimal objective value of this optimization is 0, the observed instance $\obsins$ is $(2,1,\psi)$-expansion stable.

There is always a feasible $(c,w)$ for \eqref{eqn:orig-searchprob} (see Appendix~\ref{sec:algorithm_details} for a proof), but it may change many weights and costs. Next we derive an efficiently-solvable reformulation of \eqref{eqn:orig-searchprob}.
In the next section, we find that the changes to $\hc$ and $\hw$ required to find a $(2,1,\psi)$-expansion stable instance are relatively sparse in practice.
\begin{theorem}
\label{thm:eff-solvable}
The optimization problem \eqref{eqn:orig-searchprob} can be expressed as a convex minimization problem over a polytope described by $poly(n,m,k)$ constraints. When $f(c,w) = ||(c,w)-(\hc,\hw)||_1$, \eqref{eqn:orig-searchprob} can be expressed as a linear program.
\end{theorem}

 \begin{table*}[ht]
     \centering
     \caption{Results from the output of \eqref{eqn:alg} on three stereo vision instances. More details in Appendix \ref{sec:experiments_details}.}
     \begin{tabular}{lccccc}
          Instance & Costs changed & Weights changed & (normalized) Recovery error bound & $||\hat{x}_V - \hat{x}^{MAP}_V||_1/2n$ \\
          \toprule
          ${\tt tsukuba}$ & 4.9\% & 2.3\% & 0.0518 & 0.0027\\
          ${\tt venus}$ & 22.5\% & 1.3\%  & 0.0214 & 0.0016\\
          ${\tt cones}$ & 1.2\% & 2.1\% & 0.0437 & 0.0022\\
          \bottomrule
     \end{tabular}
     \label{tbl:boundtable}
 \end{table*}
 
The instance $(G,c,w)$ is $(2,1,\psi)$-expansion stable if $x^t$ is better than every expansion $y$ of $x^t$ in every $(2,1,\psi)$-perturbation of $(c,w)$. 
Let $\calE$ be the set of all expansions of the target solution $x^t$.
Then for all $\theta'$ within a $(2,1,\psi)$-perturbation of $\theta=(c,w)$, we should have that $\dot{\theta'}{x^t} \le \min_{y\in \calE}\dot{\theta'}{x^t}$.
It is enough to check the \emph{adversarial} $(2,1,\psi)$-perturbation $\theta_{adv}$. 
This perturbation makes the target solution $x^t$ as bad as possible. 
If $x^t$ is better than all the expansions $y\in \calE$ in this perturbation, it is better than all $y\in \calE$ in every $(2,1,\psi)$-perturbation (see Appendix~\ref{sec:algorithm_details} for a proof).
We set $\theta_{adv} = (c_{adv}, w_{adv})$ as:
\begin{equation*}
    c_{adv}(u,i) = \begin{cases}
    c(u,i) + \psi & i = x^t(u),\\
    c(u,i) & \text{otherwise}.
    \end{cases}
\end{equation*}
\begin{equation*}
    w_{adv}(u,v) = \begin{cases}
    w(u,v) & x^t(u) \ne x^t(u),\\
    \frac{1}{2}w(u,v) & \text{otherwise}.
    \end{cases}
\end{equation*}
The target solution $x^t$ is fixed, so $\dot{\theta_{adv}}{x^t}$ is a linear function of the optimization variables $c$ and $w$.
 For $\alpha \in [k]$, let $\calE^{\alpha}$ be the set of $\alpha$-expansions of $x^t$. Because $\calE = \cup_{\alpha \in [k]} \calE^{\alpha}$, we have $\dot{\theta'}{x^t} \le \min_{y\in \calE}\dot{\theta'}{x^t}$ if and only if $\dot{\theta'}{x^t} \le \min_{y\in \calE^{\alpha}}\dot{\theta'}{x^t}$ for all $\alpha \in [k]$.
We can simplify the original optimization problem to:
\begin{align*}
\minimize_{c\ge 0,w\ge 0} \qquad & f(c,w)\\
\text{subject to} \qquad & \dot{\theta_{adv}}{x^t} \le \min_{y\in \calE^{\alpha}}\dot{\theta_{adv}}{y} \qquad \forall\ \alpha \in [k],
\end{align*}
$\theta_{adv}$ is a linear function of $c,w$ as defined above. 
We now use the structure of the sets $\calE^{\alpha}$ to simplify the constraints. 
The optimal value of $\min_{y\in \calE^{\alpha}}\dot{\theta_{adv}}{y}$ is the objective value of the best (w.r.t. $\theta_{adv}$) $\alpha$-expansion of $x^t$.
The best $\alpha$-expansion of a fixed solution $x^t$ can be found by solving a minimum cut problem in an auxiliary graph $G^{x^t}_{aux}(\alpha)$ whose weights depend on linearly on the objective $\theta_{adv}$, and therefore depend linearly on our optimization variables $(c,w)$ \citep[Section 4]{BoyVekZab01}. 
Therefore, the optimization problem $\min_{y\in \calE^{\alpha}}\dot{\theta_{adv}}{y}$ can be expressed as a minimum cut problem. 
Because this minimum cut problem can be written as a linear program, we can rewrite each constraint as
\begin{equation}
\label{eqn:primalconstraint}
\dot{\theta_{adv}}{x^t} \le \min_{z: A(\alpha)z=b(\alpha), z\ge 0}\dot{\theta_{adv}}{z},
\end{equation}
where $\{A(\alpha)z = b(\alpha),\ z\ge 0\}$ is the feasible region of the standard metric LP corresponding to the minimum cut problem in $G^{x^t}_{aux}(\alpha)$. The number of vertices in $G^{x^t}_{aux}(\alpha)$ and the number of constraints in $A(\alpha) z = b(\alpha)$ is $\text{poly}(m,n,k)$ for all $\alpha$.
We now derive an equivalent linear formulation of \eqref{eqn:primalconstraint} using a careful application of strong duality. The dual to the LP on the RHS is:
\begin{align*}
\maximize_{\nu}~ &~ \dot{b(\alpha)}{\nu}, 
\text{ s.t. } %\qquad & 
A(\alpha)^T\nu \le \theta_{adv}.
\end{align*}
Because strong duality holds for this linear program, we have that \eqref{eqn:primalconstraint} holds if and only if there exists $\nu$ with $A(\alpha)^T\nu \le \theta_{adv}$ such that
$\dot{\theta_{adv}}{x^t} \le \dot{b(\alpha)}{\nu}$.

This is a linear constraint in $(c,w,\nu)$. By using this dual witness trick for each label $\alpha \in [k]$, we obtain:
\begin{alignat}{2}
\label{eqn:alg}
\minimize_{c\ge 0,w\ge 0,\{\nu_{\alpha}\}} \qquad & f(c,w)\\
\text{subject to} \qquad & \dot{\theta_{adv}}{x^t} \le \dot{b(\alpha)}{\nu_{\alpha}} \qquad &&\forall\ \alpha\nonumber\\
\qquad & A(\alpha)^T\nu_{\alpha} \le \theta_{adv} &&\forall\ \alpha\nonumber.
\end{alignat}
The constraints of \eqref{eqn:alg} are linear in the optimization variables $(c,w)$ and $\nu_{\alpha}$.
The dimension of $\theta_{adv}$ is $nk + mk^2$, so there are $k(nk + mk^2 + 1)$ constraints and $nk + m + \sum_{\alpha}|b(\alpha)| = poly(m,n,k)$ variables.
Because minimization of the L1 distance $f_1(c,w)$ can be encoded using a linear function and linear constraints, \eqref{eqn:alg} is a linear program in this case. 
It is clear from the derivation of \eqref{eqn:alg} that it is equivalent to \eqref{eqn:orig-searchprob}. 
This proves Theorem \ref{thm:eff-solvable}.
This formulation \eqref{eqn:alg} can easily be input into ``off-the-shelf'' convex programming packages such as Gurobi \citep{gurobi}.

\section{Numerical results}\label{sec:experiments}
Table \ref{tbl:boundtable} shows the results of running \eqref{eqn:alg} on real-world instances of MAP inference to find nearby $(2,1,\psi)$-expansion stable instances. We study \emph{stereo vision} models using images from the Middlebury stereo dataset \citep{scharstein2002taxonomy} and Potts models from \citet{tappen2003comparison}. Please see Appendix \ref{sec:experiments_details} for more details about the models and experiments.

We find, surprisingly, that only relatively sparse changes are required to make the observed instances $(2,1,\psi)$-expansion stable with  $\psi = 1$.
On these instances, we evaluate the recovery guarantees by our bound from Theorem \ref{thm:curvature} and compare it to the actual value of the recovery error $||\hx - \hx^{MAP}||_1/2n$.
In all of our experiments, we choose the \emph{target} solution $x^t$ for \eqref{eqn:alg} to be equal to the MAP solution $\hat{x}^{MAP}$ of the observed instance.
Therefore, we find a $(2,1,\psi)$-expansion stable instance that has the same MAP solution as our observed instance.
The recovery error bound given by Theorem \ref{thm:curvature} is then also a bound for the recovery error between $\hat{x}$ and $\hat{x}^{MAP}$, because $\hx^{MAP} = x^t$. On these instances, the bounds from our curvature result (Theorem \ref{thm:curvature}) are reasonably close to the actual recovery value. However, this bound uses the property that $\hat{x}$ has good objective in the stable instance and so it is still a ``data-dependent'' bound in the sense that it uses an empirically observed property of the LP solution $\hat{x}$. In Appendix \ref{sec:experiments_details}, we show how to refine Corollary \ref{cor:deviation} to give non-vacuous recovery bounds that do not depend on $\hat{x}$.
\section{Conclusion}

We studied the phenomenon of near persistence of the local LP relaxation on instances of MAP inference in ferromagnetic Potts model. We gave theoretical results, algorithms (for finding nearby stable instances) and empirical results to demonstrate that even after a $(2,1,\psi)$-perturbation stable instance is corrupted with noise, the solution to the LP relaxation is nearly persistent i.e., it approximately recovers the MAP solution.  Our theoretical results imply that the local LP approximately recovers the MAP solution when the observed instance is close to a stable instance. Our empirical results suggest that real-world instances are very close to stable instances. These results together suggest a new explanation for the near-persistency of the solution of the local LP relaxation for MAP inference in practice.

\clearpage
\ackaccepted{We would like to thank all the reviewers for their feedback. This work was supported by NSF AitF awards CCF- 1637585 and CCF-1723344. Aravind Reddy was also supported in part by NSF CCF-1955351 and HDR TRIPODS CCF-1934931. Aravindan Vijayaraghavan was also supported by NSF CCF-1652491. }

\bibliographystyle{arxiv}
\bibliography{Paper/ref.bib}

\clearpage
\onecolumn
\aistatssupptitle{Beyond Perturbation Stability: Supplementary Material}

\appendix

\section{Preliminaries Details}\label{sec:prelim_details}

\begin{claim}\label{claim:UMLnonnegative}
For Uniform Metric Labeling, we can assume $c(u,i) \ge 0$ and $w(u,v) > 0$ without loss of generality.
\end{claim}
\begin{proof}
For problem instances where some node costs are strictly negative, let $c_{\min}$ be the minimum value among all the node costs. Consider a new problem instance where we keep the edge costs the same, but set $c'(u,i) = c(u,i) + \abs{c_{\min}}$ for all $u \in V$ and $i \in [k]$. This new problem instance has all non-negative node costs, and the optimization problem is equivalent, because we added the same constant for all solutions. This reformulation also does not affect the $(2,1)$-expansion stability or $(2,1,\psi)$-expansion stability of the instance.

Likewise, for problem instances where some edge weights are $0$, let $E_0$ be the set of all edges with $0$ edge weight. Consider a new problem instance with $E' = E \setminus E_0$, with $w(u,v)$ unchanged for $(u,v) \in E\setminus E_0$, and identical node costs. The MAP optimization problem remains the same, and the new instance $((V,E'),c,w)$ is equivalent: it has the same MAP solution, and satisfies the stability definitions if and only if the original instance does as well.
\end{proof}

\Lstarclaim*
\begin{proof}[Proof of Claim \ref{claim:Lstar}]
Recall the local LP:
\begin{alignat}{2}
  \mindot_{x}\sum_{u\in V}&\sum_i c(u,i)x_u(i) + \sum_{(u,v) \in E}&&w(u,v)\sum_{i\ne j}x_{uv}(i,j)\\
  \text{subject to:}& \sum_{i}x_u(i) = 1 &&\forall\ u\in V\label{con:norm}\\
                    & \sum_{i}x_{uv}(i,j) = x_v(j)&& \forall\ (u,v)\in E,\ j\in [k]\label{con:marg1}\\
                   & \sum_{j}x_{uv}(i,j) = x_u(i)&& \forall\ (u,v)\in E,\ i\in [k]\label{con:marg2}\\
                   & x_{uv}(i,j) \in [0,1]&&\forall\ (u,v),\ (i,j)\label{con:edge-normalize}\\
                   & x_{u}(i) \in [0,1]&& \forall\ u,\ i.
\end{alignat}
The feasible region defined by the above constraints is $L(G)$. $L^*(G) \subseteq L(G)$ is the set of points that satisfy the additional constraint that $x_{uv}(i,i) = \min(x_u(i), x_v(i))$ for all $(u,v) \in E$ and $i \in [k]$. For any feasible node variable assignments $\{x_u\}$, $L^*(G)$ is not empty: a simple flow argument\footnote{For an edge $(u,v)$, consider the bipartite graph $\tilde{G} = ((U,V), E)$, where $|U| = |V| = k$. We let $x_u(i)$ represent the supply at node $i$ in $U$, and let $x_v(i)$ represent the demand at node $j$ in $V$. Because $x_u$ and $x_v$ are both feasible, the total supply equals the total demand. $E$ contains all edges between $U$ and $V$, so we can send flow from $i\in U$ to $j \in V$ for any $(i,j)$ pair. Let $x_{uv}(i,j)$ represent this flow, and set $x_{uv}(i,i) = \min(x_u(i), x_v(i))$. For every $i$, this either satisfies the demand at node $V_i$ or exhausts the supply at node $U_i$. In each case, we can remove that satisfied/exhausted node from the graph. After this choice of $x_u(i,i)$, the total remaining supply equals the total remaining demand ($\sum_i x_u(i) - \min(x_u(i),x_v(i)) = \sum_i x_v(i) - \min(x_u(i),x_v(i))$), all supplies and demands are nonnegative, and the remaining graph $\tilde{G}'$ is a complete bipartite graph (over fewer nodes). This implies that the flow constraints \eqref{con:marg1}, \eqref{con:marg2}, \eqref{con:edge-normalize} are still feasible.}
implies that the constraints \eqref{con:marg1}, \eqref{con:marg2}, and \eqref{con:edge-normalize} are always satisfiable even when we set $x_{uv}(i,i) = \min(x_u(i), x_v(i))$. 
For all integer feasible solutions in $L(G)$, notice that $x_{uv}(i,j) = 1$ if $x_u(i) = 1$ and $x_v(j) = 1$ or $0$ otherwise. Therefore, all integer solutions satisfy this additional constraint.
Consider a $\theta$ where all edge weights are strictly positive. 
If $x$ minimizes $\dot{\theta}{x}$, $x$ must pay the minimum edge cost consistent with its node variables $x_u(i)$.
So if we fix the $x_u(i)$ portion of $x$, we know that the edge variables $x_{uv}$ of $x$ are a solution to:
\[ \min_{x \in L(G)} \sum_{(u,v) \in E}w(u,v)\sum_{i\ne j}x_{uv}(i,j). \]
Notice that since we have fixed the node variables $x_u(i)$, there is no interaction between the $x_{uv}$ variables across different edges. 
So we can minimize this objective by minimizing each individual term $w(u,v)\sum_{i\ne j}x_{uv}(i,j)$. 
Since $w_{uv} > 0$ for all edges, we need to minimize $\sum_{i\ne j}x_{uv}(i,j)$. 
Notice that for every edge $(u,v) \in E$, we get that $\sum_{i} \sum_{j} x_{uv}(i,j) = 1$ by substituting $x_u(i)$ in constraint~\ref{con:norm} with $\sum_{j} x_{uv}(i,j)$ from constraint~\ref{con:marg2}. 
Therefore $\sum_{i\ne j} x_{uv}(i,j) = 1 - \sum_{i}x_{uv}(i,i)$. 
Thus, minimizing $\sum_{i\ne j}x_{uv}(i,j)$ is  the same as maximizing $\sum_{i}x_{uv}(i,i)$. 
And the maximizing choice for $x_{uv}(i,i) = \min (x_u(i), x_v(i))$ due to constraints~\ref{con:marg1} and \ref{con:marg2}.
\end{proof}

\section{Expansion Stability details}\label{sec:expansion-stability_details}

\begin{figure}[tb]
  \centering
\begin{subfigure}{.5\linewidth}
  \centering
  \scalebox{1.0}{  
  \tikzstyle{vertex}=[circle, draw=black, very thick, minimum size=5mm]
  \tikzstyle{edge} = [draw=black, line width=1]
  \tikzstyle{weight} = [font=\normalsize]
  \begin{tikzpicture}[scale=2,auto,swap]
    \foreach \pos /\name in {{(0,0)}/u,{(1,0)}/w,{(0.5,0.75)}/v}
    \node[vertex](\name) at \pos{$\name$};
    \foreach \source /\dest /\weight in {u/w/1+\epsilon}
    \path[edge] (\source) -- node[weight] {$\weight$} (\dest);
    \foreach \source /\dest /\weight/\pos in {u/v/1+\epsilon/{above left}, v/w/1+\epsilon/{above right}}
    \path[edge] (\source) -- node[weight, \pos] {$\weight$} (\dest);
  \end{tikzpicture}
  }
\end{subfigure}%
\begin{subfigure}{.5\linewidth}
  \centering
  \scalebox{1.0}{
\begin{tabular}{|l|ccc|}
\hline
\multicolumn{1}{|c|}{\textbf{Node}} & \multicolumn{3}{|c|}{\textbf{Costs}} \\
\hline
u & .5 & $\infty$        & $\infty$      \\
\hline
v & 1 & 0 & $\infty$\\
\hline
w & 1 & $\infty$ & 0\\
\hline
\end{tabular}
}
\end{subfigure}
  \caption{$(2,1)$-expansion stable instance that is not $(2,1)$-stable. In the original instance (shown left), the optimal solution labels each vertex with label 1, for an objective of $2.5$. The adversarial $(2,1)$-perturbation for this instance replaces all the edge weights of $1+\epsilon$ with $(1+\epsilon)/2$. In this perturbed instance, the optimal solution labels $(u,v,w) \rightarrow (1,2,3)$. This has a node cost of 0.5 and an edge cost of $(3+3\epsilon)/2$, for a total of $2+3\epsilon/2 < 2.5$. Since the original solution is not optimal in the perturbed instance, this instance is not $(2,1)$-perturbation stable. However, note that the only expansions of the original solution (which had all label 1) that have non-infinite objective are $(u,v,w) \rightarrow (1,2,1)$ and $(u,v,w) \rightarrow (1,1,3)$. These each have objective $2.5 + \epsilon$, which is strictly greater than the perturbed objective of the original solution. Therefore, this instance is $(2,1)$-expansion stable.}
\label{fig:counter1-apdx}
\end{figure}
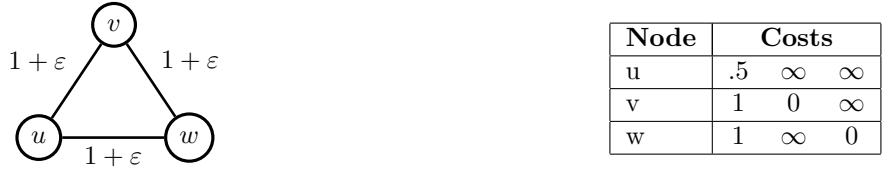

\begin{claim}\label{exp_worst}
An instance $(G,w,c)$ is $(2,1)$-expansion stable iff the MAP solution $\bx$ is strictly better than all its expansions in the adversarial perturbation $\theta_{adv}$. That is, for all $x \in \calE_{\bx}, \dot{\theta_{adv}}{x} > \dot{\theta_{adv}}{\bx}$ where $\theta_{adv}$ has the same node costs $c$ but has weights $
    w_{adv}(u,v) = \begin{cases}
    \frac{1}{2}w(u,v) & \bx(u) = 
    \bx(v)\\
    w(u,v) & \bx(u) \ne \bx(v).
\end{cases}
$
\end{claim}
\begin{proof}
Consider $\theta' = (c,w')$, any valid $(2,1)$-perturbation of $\theta = (c,w)$ i.e. for every edge $(u,v) \in E, \frac{w(u,v)}{2} \leq w'(u,v) \leq w(u,v)$. For any valid labeling $x$, let $E_x$ represent the edges cut by $x$. Then, for any $x$ which is an expansion of $\bx$ i.e. $x \in \calE_{\bx}$,
\begin{align*}
    \dot{\theta'}{x} - \dot{\theta'}{\bx} &= \sum_{u \in V} c(u,x(u)) - c(u,\bx(u)) + \sum_{(u,v) \in E_x} w'(u,v) - \sum_{(u,v) \in E_{\bx}} w'(u,v)
    \\&= \sum_{u \in V} c(u,x(u)) - c(u,\bx(u)) + \sum_{(u,v) \in E_x \setminus E_{\bx}} w'(u,v) - \sum_{(u,v) \in E_{\bx}\setminus E_x} w'(u,v)
    \\&= \dot{\theta_{adv}}{x} - \dot{\theta_{adv}}{\bx} + \sum_{(u,v) \in E_x \setminus E_{\bx}} w'(u,v) -w_{adv}(u,v)  + \sum_{(u,v) \in E_{\bx}\setminus E_x} w_{adv} - w'(u,v)
    \\&= \dot{\theta_{adv}}{x} - \dot{\theta_{adv}}{\bx} + \sum_{(u,v) \in E_x \setminus E_{\bx}} w'(u,v) -\frac{w(u,v)}{2}  + \sum_{(u,v) \in E_{\bx}\setminus E_x} w(u,v) - w'(u,v)
\end{align*}
Since $w'$ is a valid $(2,1)$-perturbation, $w'(u,v) \ge w(u,v)/2$ and $w'(u,v) \le w(u,v)$. Therefore, for any valid $(2,1)$-perturbation $\theta'$, we have 
\[
\dot{\theta'}{x} - \dot{\theta'}{\bx} \geq \dot{\theta_{adv}}{x} - \dot{\theta_{adv}}{\bx}.
\]

If the instance is $(2,1)$-expansion stable, then certainly $\dot{\theta_{adv}}{x} > \dot{\theta_{adv}}{\bx}$ for all $x\in \calE^{\bx}$, since $\theta_{adv}$ is a valid $(2,1)$-perturbation of $\theta$. If the instance is not $(2,1)$-expansion stable, there exists a $\theta'$ and an $x\in \calE^{\bx}$ for which $\dot{\theta'}{x} - \dot{\theta'}{\bx} \le 0$. But the above inequality then implies that $\dot{\theta_{adv}}{x} - \dot{\theta_{adv}}{\bx} \le 0$ as well. This gives both directions.
\end{proof}

This claim shows that to check whether an instance is $(2,1)$-expansion stable, it is sufficient to check that the MAP solution is strictly better than all its expansions in the adversarial perturbation $\theta_{adv}$. We don't need to verify that this condition is satisfied in \textit{every} valid $(2,1)$-perturbation. Because the optimal expansion of $\bx$ in the instance with objective $\theta_{adv}$ can be computed efficiently, this claim also implies that $(2,1)$-expansion stability can be efficiently checked once the MAP solution $\bx$ is known.

\begin{claim}\label{exp_weak}
$(2,1)$-expansion stability is strictly weaker than $(2,1)$-perturbation stability. 
\end{claim}
\begin{proof}
Figure \ref{fig:counter1} gives an instance of uniform metric labeling that is $(2,1)$-expansion stable but not $(2,1)$-perturbation stable. Here, $0 < \epsilon < 1/3$.
\end{proof}

\lptes*
\begin{proof}
First, we note that for any $x \in L^*(G)$, the objective value of the local LP can be written in a form that depends only on the node variables $x_V$. The objective term corresponding to the edges
\begin{align*}
    \sum_{(u,v) \in E}w(u,v)\sum_{i\ne j}x_{uv}(i,j) &= \sum_{(u,v) \in E}w(u,v)\rbr{\sum_{i,j}x_{uv}(i,j) - \sum_{i}x_{uv}(i,i)}
    \\&= \sum_{(u,v) \in E}w(u,v)\rbr{1 - \sum_{i}x_{uv}(i,i)} = \sum_{(u,v) \in E}w(u,v)\rbr{1 - \sum_{i}\min(x_u(i),x_v(i))}
    \\&= \sum_{(u,v) \in E}w(u,v)\rbr{1 - \sum_{i}\rbr{\frac{x_u(i) + x_v(i)}{2} - \frac{\abs{x_u(i) - x_v(i)}}{2}}}
    \\&= \sum_{(u,v) \in E}w(u,v)\rbr{\frac{1}{2}\sum_{i}\abs{x_u(i) - x_v(i)}}
\end{align*}

Here we used the definition of $L^*(G)$ and the facts that $\sum_i x_u(i) = 1$ for all $(u,i)$ and $\sum_{j}x_{uv}(i,j) = x_u(i)$ for all $(u,v)\in E,\ i\in [k]$.

Thus, for any $x \in L^*(G)$, the objective of the local LP can be written as \[\sum_{u\in V}\sum_i c(u,i)x_u(i) + \sum_{(u,v) \in E}w(u,v)d(u,v)\]
where $d(u,v) \coloneqq \frac{1}{2}\sum_{i}\abs{x_u(i) - x_v(i)}$. This is the objective function of another LP relaxation for uniform metric labeling called the ``metric LP'', which is equivalent to the local LP \citep{archer2004approximate}. Note that both $\bar{x}$ and $\hx$ are in $L^*(G)$ by Claim \ref{claim:Lstar}. Therefore, the objective function can be written in the above form for both of them.

In the next section, we introduce a rounding algorithm and prove some guarantees for the random solutions $h$ output by it. We then use these guarantees to show an upper bound on the expected cost of these random solutions in a perturbed instance of the problem. Finally, we use this upper bound to prove that $\hx = \bx$.

\subsection{$\epsilon$-close rounding:}
Given any feasible solution $x \in L(G)$ and a valid labeling $\bar{x}$, we construct a related feasible solution $x'$ which is $\epsilon$-close to $\bar{x}$ in the $\ell_{\infty}$-norm i.e. $\norm{x' - \bar{x}}_{\infty} \leq \epsilon$:
\begin{equation}\label{eqn:eps-close}
    x' = \epsilon x + (1-\epsilon)\bar{x},
\end{equation}
where $\epsilon < 1/k$ and we have identified the labeling $\bar{x}:V\to\mathcal{L}$ with its corresponding vertex of the marginal polytope (a vector in $\{0,1\}^{nk+mk^2}$). We consider the following rounding algorithm applied to $x'$, which is a modified version of the $\epsilon$-close rounding algorithm used in \citet{LanSonVij18}:
\begin{algorithm}[H]
   \caption{$\epsilon$-close rounding}
   \label{alg:rounding_algorithm}
\begin{algorithmic}[1]
   \STATE Choose $i \in \{1,\dots,k\}$ uniformly at random.
   \STATE Choose $r \in (0,1/k)$ uniformly at random.
   \STATE Initialize labeling $h: V\to [k]$.
   \FOR{each $u \in V$}
        \IF {$x'_u(i) > r$}
            \STATE Set $h(u) = i$.
        \ELSE
            \STATE Set $h(u) = \bx(u)$
        \ENDIF 
   \ENDFOR
   \STATE \textbf{Return} $h$
\end{algorithmic}
\end{algorithm}

\begin{lemma}[Rounding guarantees]
\label{lemma:rounding-guar}
Given any $x'$ constructed using \eqref{eqn:eps-close}, the labeling $h$ output by Algorithm \ref{alg:rounding_algorithm} satisfies the following guarantees:
\begin{align*}
    \prob{h(u) = i} &= x'_u(i) &\forall\ u\in V, i\in [k]\\
    \prob{h(u) \ne h(v)} &\le 2d(u,v) &\forall\ (u,v) \in E : \bx(u) = \bx(v)\\
    \prob{h(u) = h(v)} &= (1-d(u,v)) &\forall\ (u,v) \in E : \bx(u) \ne \bx(v),
\end{align*}
where $d(u,v) = \frac{1}{2}\sum_i|x'_u(i) - x'_v(i)|$ is the edge separation of the constructed feasible point $x'$.
\end{lemma}
\begin{proof}[Proof of Lemma \ref{lemma:rounding-guar} (rounding guarantees)]
First, fix $u\in V$ and a label $i\ne \bx(u)$. We output $h(u) = i$ precisely when $i$ is chosen and $0 < r < x'_u(i)$, which occurs with probability $\frac{1}{k}\frac{x'_u(i)}{1/k} = x'_u(i)$ (we used here that $x'_u(i) \le \epsilon < 1/k$ for all $i\ne \bx(u)$). Now we output $h(u)=\bx(u)$ with probability $1-\sum_{j\ne \bx(u)}\prob{h(u) = j} = 1-\sum_{j\ne \bx(u)}x'_u(j) = x'_u(\bx(u))$, since $\sum_{i}x'_u(i) = 1$. This proves the first guarantee.

For the second, consider an edge $(u,v)$ not cut by $\bx$, so $\bx(u) = \bx(v)$. Then $(u,v)$ is cut by $h$ when some $i \ne \bx(u)$ is chosen and $r$ falls between $x'_u(i)$ and $x'_v(i)$. This occurs with probability \[\frac{1}{k}\sum_{i\ne \bx(u)}\frac{\max(x'_u(i), x'_v(i))-\min(x'_u(i), x'_v(i))}{1/k} = \sum_{i\ne \bx(u)} |x'_u(i) - x'_v(i)| \le 2d(u,v).\]

Finally, consider an edge $(u,v)$ cut by $\bx$, so that $\bx(u) \ne \bx(v)$. Here $h(u) = h(v)$ if some $i,r$ are chosen with $r < \min(x'_u(i), x'_v(i))$. We have $r < \min(x'_u(i), x'_v(i))$ with probability $\frac{\min(x'_u(i), x'_v(i))}{1/k}$. Note that this is still valid if $i=\bx(u)$ or $i=\bx(v)$, since only one of those equalities can hold.
So we get \[ \frac{1}{k}
\sum_{i} \frac{\min(x'_u(i),x'_v(i))}{1/k}  = \frac{1}{2}\left(\sum_i x'_u(i) + x'_v(i) - |x'_u(i) - x'_v(i)|\right) = 1 - d(u,v),
\]
where we used again that $\sum_i x'_u(i) = 1$.
\end{proof}
Given these rounding guarantees, we can relate the expected cost difference between $h$ and $\bx$ in a perturbation of the original instance to the cost difference between $x$ and $\bx$ in the original instance. We are only interested in the case when $x \in L^*(G)$ and so the objective function $f(x) = \sum_{u\in V}\sum_i c(u,i)x_u(i) + \sum_{(u,v) \in E}w(u,v)d(u,v)$.

\subsection{Using the rounding guarantees}
\begin{lemma}\label{lem:Ahupp}
Given an integer solution $\bx$, a feasible LP solution $x \in L^*(G)$, and a random output $h$ of Algorithm \ref{alg:rounding_algorithm} on an input $x' = \epsilon x + (1-\epsilon)\bar{x}$, define
\[
    w'(u,v) = \begin{cases}
    \frac{1}{2}w(u,v) & \bx(u) = 
    \bx(v)\\
    w(u,v) & \bx(u) \ne \bx(v)
\end{cases}
\]
and let $f'(y) = \sum_{u\in V}\sum_i c(u,i)y_u(i) + \sum_{(u,v) \in E}w'(u,v)d(y,u,v)$ be the objective in the instance with the original costs, but using weights $w'$. Here $d(y,u,v) = \frac{1}{2}\sum_i|y_u(i) - y_v(i)|$. Let $A_h \coloneqq f'(h) - f'(\bx)$ be the difference in this perturbed objective between $h$ and $\bx$. Then,
\[
 \E[A_h] = \E{[f'(h) - f'(\bx)]} \le \epsilon \cdot \rbr{f(x) - f(\bx)}.
\]
\end{lemma}
\begin{proof}
\begin{align*}
\E{[f'(h) - f'(\bx)]} &= \sum_{u\in V}\sum_{i}c(u,i)\prob{h(u) = i} - \sum_{u\in V}c(u,\bx(u)) + \sum_{uv : \bx(u) = \bx(v)} w'(u,v)\prob{h(u) \ne h(v)} \\
\qquad\qquad &- \sum_{uv : \bx(u) \neq \bx(v)} w'(u,v)\prob{h(u) = h(v)}\\
&= \sum_u\sum_ic(u,i)x'_u(i) - \sum_uc(u,\bx(u)) + \sum_{uv: \bx(u)=\bx(v)}2w'(u,v)d(x',u,v) \\
&- \sum_{uv : \bx(u) \ne \bx(v)}w'(u,v)(1-d(x',u,v))\\
&= \sum_u\sum_ic(u,i)x'_u(i) - \sum_uc(u,\bx(u)) + \sum_{uv\in E}w(u,v)d(x',u,v)- \sum_{uv : \bx(u) \ne \bx(v)}w(u,v)\\
&= f(x') - f(\bx).
\end{align*}
where the second-to-last equality used the definition of $w'$ (note that $w'$ is identical to the worst-case perturbation $w_{adv}$ for $\bx$). Because $f$ is convex (in particular, $d(x,u,v)$ is convex in $x$), we have $f(x') \le \epsilon f(x) + (1-\epsilon)f(\bx)$. Therefore,
\[
\E{[f'(h) - f'(\bx)]} \le \epsilon f(x) + (1-\epsilon)f(\bx) - f(\bx) = \epsilon(f(x) - f(\bx)),
\]
which is what we wanted.
\end{proof}

\subsection{Final proof of Theorem~\ref{thm:lptes}:}\label{sec:finlpt}
%We prove Theorem~\ref{thm:lptes} using Lemma~\ref{lem:Ah}.
Suppose the local LP solution $\hx$ is not the same as the MAP solution $\bx$ i.e. $\hx \neq \bx$. Consider $x' = \epsilon\hat{x} + (1-\epsilon)\bar{x}$ where $0< \epsilon < 1/k$ (see equation~\eqref{eqn:eps-close}).
Let $h$ be the random integer solution output by using Algorithm~\ref{alg:rounding_algorithm} on $x'$. By Lemma~\ref{lem:Ahupp}, we have 
\[ \E{[f'(h) - f'(\bx)]} \leq \epsilon \cdot \rbr{f(\hx) - f(\bx)} \]
%Here $f'$ is the worst-case $(2,1)$-perturbation for the MAP solution $\bx$.

We note that any solution $h$ that we get from rounding $x'$ is either $\bx$ or an expansion move of $\bx$. This is because we pick only a single label $i$ in step 1 of Algorithm~\ref{alg:rounding_algorithm} and label all vertices $u$ either $i$ or $\bx(u)$. 
Therefore, for the $i$ picked in step 1, $h$ is an $i$-expansion of $\bx$ if $h \neq \bx$.
\begin{align*}
    \E{[f'(h) - f'(\bx)]} &= \E{[f'(h) - f'(\bx)|h\neq \bx]}\ \Pr[h \neq \bx] + \E{[f'(h) - f'(\bx)|h = \bx]}\ \Pr[h = \bx]
    \\&= \E{[f'(h) - f'(\bx)|h\neq \bx]}\ \Pr[h \neq \bx]
\end{align*}
Since $(G,c,w)$ is a $(2,1)$-expansion stable instance, we know that $f'(h) > f'(\bx)$ when $h \neq \bx$ since all $h$ in the support of the rounding (other than $\bx$) are expansion moves of $\bx$ and we get $f'$ by a valid $(2,1)$-perturbation of $(G,c,w)$.
Therefore, $\E{[f'(h) - f'(\bx)|h\neq \bx]} > 0$. We also have that $\Pr[h \neq \bx] > 0$ since we assumed that $\hx \neq \bx$. Therefore, $\E{[f'(h) - f'(\bx)]} > 0$. But we know that $f(\hx) - f(\bx) \leq 0$ since $\hx$ is the minimizer of $f(x)$ among all feasible $x \in L(G)$. So Lemma \ref{lem:Ahupp} implies $\E{[f'(h) - f'(\bx)]} \le 0$. Thus we have a contradiction and so the local LP solution $\hx$ has to be the same as the MAP solution $\bx$.

\end{proof}

\section{Stability and Curvature around MAP solution: details}\label{sec:stable-curvature_details}
\curvature*

Here, we provide two proofs for this theorem, one deals directly with the local LP relaxation and the other uses the dual of the relaxation. The dual proof is more general than the primal proof as it works for all $x \in L(G)$, not just for those in $L^*(G)$.

\subsection{Primal-based proof}
\begin{proof}
For any $x \in L^*(G)$, consider a feasible solution $x'$ which is $\epsilon$-close to $\bx$ constructed using Equation~\ref{eqn:eps-close} i.e. $x' = \epsilon x + (1-\epsilon)\bar{x}$. Let $h$ be the random solution output by Algorithm~\ref{alg:rounding_algorithm} on $x'$.

\begin{lemma}[Bound for $\E{[B_h]}$]\label{lem:Bh}
For any $h$ in the support of the rounding of $x' = \epsilon x + (1-\epsilon)\bar{x}$, let us define $B_h$ to be the number of vertices which it labels differently from $\bx$. In other words, it is the number of vertices which are misclassified by $h$ i.e. $B_h \coloneqq \sum_{u \in V} \indicator[h(u) \ne \bx(u)]$. Then, 
\[ \E{[B_h]} = \epsilon \sum_{u \in V} \frac{1}{2} \norm{x_u - \bx_u}_1 \]
\end{lemma}
\begin{proof}
\begin{align*}
    \E[B_h] &= \sum_{u \in V} \E[\indicator[h(u) \neq \bx(u)]] = \sum_{u \in V} \prob{h(u) \neq \bx(u)} = \sum_{u \in V} 1 - \prob{h(u) = \bx(u)} 
    \\&= \sum_{u \in V} 1 - x'_u(\bx(u)) = \sum_{u \in V} 1 - \rbr{\epsilon x_u(\bx(u)) + (1-\epsilon)} = \sum_{u \in V} \epsilon \rbr{1 -  x_u(\bx(u))}
    \\&= \epsilon \sum_{u \in V} \frac{1}{2} \rbr{1 - x_u(\bx(u)) + \sum_{i \neq \bx(u)} x_u(i) } = \epsilon \sum_{u \in V} \frac{1}{2} \norm{x_u - \bx_u}_1
\end{align*}
Here, we used the fact that for all $u \in V, \bx_u(\bx(u)) = 1 \text{ and } \bx_u(i) = 0\ \forall\ i \neq \bx(u)$ and since $x$ is a feasible solution to the LP, it satisfies $\sum_{i \ne \bx(u)} x_u(i) = 1 - x_u(\bx(u))$ for all $u \in V$.
\end{proof}

\begin{lemma}[Lower bound for $A_h$ using $(2,1,\psi)$-expansion stability]\label{lem:Ahlow}
If $(G,w,c)$ is a $(2,1,\psi)$-expansion stable instance, then for any $h$ in the support of the rounding of $x' = \epsilon x + (1-\epsilon)\bar{x}$,
 \[A_h \ge \psi \cdot B_h\]
where $A_h = f'(h) - f'(\bx)$ and $f'$ is the objective in the instance $(G,c,w')$ where $w'$ is the worst $(2,1)$ perturbation for $\bx$ i.e. \[
    w'(u,v) = \begin{cases}
    \frac{1}{2}w(u,v) & \bx(u) = 
    \bx(v)\\
    w(u,v) & \bx(u) \ne \bx(v)
\end{cases}
\]
\end{lemma}

\begin{proof}
Note that $A_h$ here is the same as the one defined for Lemma~\ref{lem:Ahupp}.
Since the instance $(G,c,w)$ is $(2,1,\psi)$-expansion stable, we know that $(G,c',w)$ should be $(2,1)$-expansion stable for all $c'$ such that $c \leq c' \leq c + \psi \cdot \mathbf{1}$. Consider the worst $c'$ for $\bx$ i.e.  $c'(u,i) = \begin{cases}
    c(u,i) + \psi & i = \bx(u)\\
    c(u,i) & i \ne \bx(u)
\end{cases}$. 
Let $f''$ be the objective in the instance $(G,c',w')$. As discussed in section~\ref{sec:finlpt}, we know that any $h \neq \bx$ in the support of the rounding is an expansion move of $\bx$. Therefore, for any $h \neq \bx$ in the support of the rounding of $x'$,

\begin{align*}
    & f''(h) - f''(\bx) > 0 \implies \sum_{u \in V} c'(u,h(u)) - c'(u,\bx(u))  + \sum_{(u,v):h(u)\neq h(v)} w'(u,v) - \sum_{(u,v):\bx(u)\neq \bx(v)} w'(u,v) > 0
    \\&\implies \sum_{u \in V} c(u,h(u)) - \rbr{c(u,\bx(u)) + \psi \cdot \indicator[h(u) \ne \bx(u)] } + \sum_{(u,v):h(u)\neq h(v)} w'(u,v) - \sum_{(u,v):\bx(u)\neq \bx(v)} w'(u,v) > 0
    \\&\implies f'(h) - f'(g) > \psi \cdot \sum_{u \in V} \indicator[h(u) \ne \bx(u)] \implies A_h > \psi \cdot B_h.
\end{align*}
This is true for all $h \neq \bx$ in the support of the rounding of $x'$. When $h = \bx$, we have $A_h = B_h = 0$. Therefore for all $h$ in the support of the rounding of $x'$, we have that $A_h \geq \psi \cdot B_h$.  \end{proof}

\subsection{Final proof of Theorem~\ref{thm:curvature}:}
We use the Lemmas \ref{lem:Ahupp}(upper bound for $\E[A_h]$), \ref{lem:Ahlow}(lower bound for $A_h$), and \ref{lem:Bh}(bound for $\E[B_h]$) to prove Theorem~\ref{thm:curvature}. For all $h$ in the support of rounding of $x$, $A_h \geq \psi \cdot B_h$. Also, 
\[ \E[A_h] \leq \epsilon \rbr{f(x) - f(\bx)},\ \E[B_h] = \epsilon \sum_{u \in V} \frac{1}{2} \norm{x_u - \bx_u}_1  \]

Suppose that $\norm{x - \bx}_1 > \tau \cdot \rbr{f(x) - f(\bx)}$. Then,

\[ \frac{\E[A_h]}{\E[B_h]} \leq \frac{f(x) - f(\bx)}{\sum_{u \in V} \frac{1}{2} \norm{x_u - \bx_u}_1} < \frac{2}{\tau} \]

But since $A_h \geq \psi \cdot B_h$ for every $h$ in the rounding of $x$, we get that $\dfrac{\E[A_h]}{\E[B_h]} \geq \psi$.
\\Setting $\tau = \frac{2}{\psi}$, we get a contradiction and thus we get, \[ \frac{1}{2}\norm{x - \bx}_1 \leq \frac{1}{\psi}\cdot \rbr{f(x) - f(\bx)} = \frac{1}{\psi} \cdot \rbr{\dot{\theta}{x}-\dot{\theta}{\bx}} \]
\end{proof}

\deviation*
\begin{proof}[Proof of Corollary \ref{cor:deviation}]
First, we note that for the nearby stable instance, the MAP and the local LP solutions are the same due to Theorem~\ref{thm:lptes}. Therefore, for any feasible solution $x \in L(G)$, $ \dot{\bar{\theta}}{x} \ge \dot{\bar{\theta}}{\bx}$. In particular, this implies that $\dot{\bar{\theta}}{\hx} \ge \dot{\bar{\theta}}{\bx}$ and $\dot{\bar{\theta}}{\hx^{MAP}} \ge \dot{\bar{\theta}}{\bx}$ since $\hx, \hx^{MAP}$ are also feasible solutions. Remember that we defined $d(\hat{\theta},\bar{\theta}) \coloneqq \sup_{x \in L^*(G)} \abs{ \dot{\hat{\theta}}{x}  - \dot{\bar{\theta}}{x}}$. Therefore,
\[ \dot{\bar{\theta}}{\hx} \leq \dot{\hat{\theta}}{\hx} + d(\hat{\theta},\bar{\theta}) \leq \dot{\hat{\theta}}{\bx} + d(\hat{\theta},\bar{\theta}) \leq \dot{\bar{\theta}}{\bx} + 2d(\hat{\theta},\bar{\theta}). \]
The first and third inequalities hold due to the definition of $d(\hat{\theta},\bar{\theta})$. The second inequality follows from the fact that $\hx$ is the minimizer for $\dot{\hat{\theta}}{x}$ among $x \in L(G)$. Thus, $0 \leq \dot{\bar{\theta}}{\hx} - \dot{\bar{\theta}}{\bx} \leq 2d(\hat{\theta},\bar{\theta})$. From Theorem~\ref{thm:curvature}, we get
$\frac{1}{2}\norm{\hx_V - \bx_V}_1 \le \frac{2d(\hat{\theta},\bar{\theta})}{\psi}$. Thus,
\begin{align*}
    \frac{1}{2}\norm{\hx_V - \hx_V^{MAP}}_1 &\leq \frac{1}{2}\norm{\hx_V - \bx_V}_1 + \frac{1}{2}\norm{\hx_V^{MAP} - \bx_V}_1 \\&\leq \frac{2d(\hat\theta,\bar{\theta})}{\psi} + \frac{1}{2}\norm{\hx_V^{MAP} - \bx_V}_1. 
\end{align*}
\end{proof}

\subsection{Dual-based proof}%of Theorem \ref{thm:curvature}
Here we provide an alternate proof of the curvature result using the dual of the local LP relaxation.
First, we show that the curvature bound is related to the \emph{dual margin} of the instance. 
Then we show that $(2,1,\psi)$-expansion stability implies that the dual margin is at least $\psi$.
Throughout this section, we assume the local LP solution $\hat{x}$ is unique and integral (as guaranteed, for example, by $(2,1)$-expansion stability), so $\hat{x} = \bar{x}$.

Relaxing the
local LP's marginalization constraints in both directions for each edge, we obtain
the following Lagrangian for the local LP:
\[
L(\delta,x) = \sum_{u}\sum_i\left(\theta_u(i) + \sum_{v \in N(u)}\delta_{uv}(i)\right)x_u(i) + \sum_{uv}\sum_{ij}\left(\theta_{uv}(i,j) - \delta_{uv}(i) - \delta_{vu}(j)\right)x_{uv}(i,j)
\]
where each $x_u$ is constrained to be in the $(k-1)$-dimensional simplex, and each $x_{uv}$ the $k^2-1$-dimensional simplex (i.e., the normalization constraints remain). There are no constraints on the dual variables $\delta$. Observe that for any $\delta$ and any primal-feasible $x$, $L(\delta,x) = \langle \theta, x\rangle$. This gives rise to the \emph{reparametrization} view: for a fixed $\delta$, define $\theta^{\delta}_u(i) = \theta_u(i) + \sum_{v\in N(u)}\delta_{uv}(i)$, and $\theta^{\delta}_{uv}(i,j) = \theta_{uv}(i,j) - \delta_{uv}(i) - \delta_{vu}(j)$. Then $L(\delta, x) = \langle \theta^{\delta}, x\rangle$. This will allow us to define equivalent primal problems with simpler structure than the original. $L(\delta,x)$ also gives the dual function:
\[
D(\delta) = \min_x L(\delta,x) = \sum_{u}\min_{i}\left(\theta_u(i) + \sum_{v \in N(u)}\delta_{uv}(i)\right) + \sum_{uv}\min_{i,j}\left(\theta_{uv}(i,j) - \delta_{uv}(i) - \delta_{vu}(j)\right).
\]
A dual point $\delta$ is a dual \emph{solution} if $\delta \in \argmax_{\delta'} D(\delta')$. Theorem \ref{thm:lptes} implies that the local LP has a unique, integral solution when the instance is $(2,1,\psi)$-expansion stable. 
\citet[Theorem 1.3]{sontag2011introduction} show that this implies the existence of a dual solution $\delta^*$ that is \emph{locally decodable} at all nodes $u$: for each $u$, $\argmin_i \theta^{\delta^*}_u(i)$ is unique, and moreover, the edge and node dual subproblems agree:
\begin{equation}
\label{eqn:dual-subproblems-agree}
\left(\argmin_i \theta^{\delta^*}_u(i), \argmin_j \theta^{\delta^*}_v(j)\right) \in \argmin_{i,j}\theta_{uv}^{\delta^*}(i,j).
\end{equation}
In this case, the primal solution defined by ``decoding'' $\delta^*$, $x(u) = \argmin_i \theta^{\delta^*}_u(i)$, is the MAP solution \citep{sontag2011introduction}.

For locally decodable $\delta^*$, we define the \emph{node margin} $\psi_u(\delta^*)$ at a node $u$ as: \[
\psi_u(\delta^*) = \min_{i\ne \argmin_j \theta^{\delta^*}(j)} \theta^{\delta^*}(i) - \min_j \theta^{\delta^*}(j).
\]
This is the difference between the optimal reparametrized node cost at $u$ and the next-smallest cost.
Local decodability of $\delta$ is the property that $\psi_u(\delta) > 0$ for every $u$. 

Together with \eqref{eqn:dual-subproblems-agree}, the following lemma implies that we need only consider locally decodable dual solutions where the optimal primal solution pays zero edge cost.
\begin{lemma}[Dual edge ``removal'']\label{lemma:edge-removal}
  Given a locally decodable dual solution $\delta$, we can transform
  it to a locally decodable dual solution $\delta'$ that satisfies
  $\min_{i,j}\theta_{uv}^{\delta'}(i,j) = 0$ and has the same (additive) margin at every node.
\end{lemma}
\begin{proof}
  Fix an edge $(u,v)$, and consider any pair $i^*,j^*$ in $\argmin_{i,j}\theta_{uv}^{\delta}(i,j)$. Put $\theta^\delta_{uv}(i^*,j^*) = \theta_{uv}(i^*,j^*) + \epsilon$ for $\epsilon \in \mathbb{R}$. Now define $\delta'_{uv}(i) = \delta_{uv}(i) - \epsilon$ for all $i$ (or, equivalently, $\delta'_{vu}(j) = \delta_{vu}(j) - \epsilon$ for all $j$). Because we changed $\theta^{\delta}_u(i)$ by a constant for each $i$, local decodability is preserved and the additive \emph{margin of local decodability} is not changed. We incurred a change of $+\epsilon$ in the dual objective of $\delta$ from the edge term $\min_{i,j} \theta_{uv}^{\delta'}(i,j)$, and a $-\epsilon$ in the objective from the decrease in the node term $\min_i \theta^{\delta'}_u(i)$, so $\delta'$ is still optimal. We can repeat this process for every edge $(u,v)$.
\end{proof}
Lemma \ref{lemma:edge-removal} implies that when $(x^*,\delta^*)$ is a
pair of primal/dual optima and $\delta^*$ is locally decodable, we can assume that $L(x^*,\delta^*) = \sum_u\theta^{\delta^*}_u(x^*_u)$, where we overload notation to define $x^*_u$ to be the label for which $x^*_u(i) = 1$. That is, the primal optimum pays no edge cost in the problem reparametrized by the dual opt $\delta^*$. Finally, Lemma \ref{lemma:edge-removal} implies that we can always assume that $\theta_{uv}^{\delta^*}(i,j) \ge 0$ for all $(u,v),\ (i,j)$. Therefore, if there is any locally decodable dual solution, and the primal LP solution is integral and unique, we may assume there exists a locally decodable dual solution $\delta$ such that $\bar{x}(u) = \argmin_i\theta^{\delta}(i)$, $(\bx(u), \bx(v)) \in \argmin_{i,j}\theta^{\delta}_{uv}(i,j)$, $\theta^{\delta}_{uv}(\bx(u),\bx(v)) = 0$, and $\theta^{\delta}_{uv}(i,j) \ge 0$.

\begin{lemma}[Dual margin implies curvature around $\bx$]
For an instance with objective $\theta$ and MAP solution $\bx$, assume there exists a locally decodable dual solution $\delta$ such that $\bar{x}(u) = \argmin_i\theta^{\delta}(i)$, $(\bx(u), \bx(v)) \in \argmin_{i,j}\theta^{\delta}_{uv}(i,j)$, $\theta^{\delta}_{uv}(\bx(u),\bx(v)) = 0$, and $\theta^{\delta}_{uv}(i,j) \ge 0$. Additionally, let $\psi(\delta) = \min_u\psi_u(\delta)$ be the smallest node margin. Note that $\psi(\delta) > 0$ because $\delta$ is locally decodable.
Then for any $x\in L(G)$,
\[
\frac{1}{2}||x_V-\bx_V||_1 \le \frac{\dot{\theta}{x-\bx}}{\psi(\delta)}
\]
\end{lemma}
\begin{proof}
  Let $\Delta = \dot{\theta}{x - \bx}$. 
  Since $x$ and $\bx$ are both primal-feasible, we have $L(x,\delta) = \dot{\theta}{x}$ and $L(\bx, \delta) = \dot{\theta}{\bx}$. Therefore,
  \begin{equation}\label{eqn:lagrangian-ineq}
  L(x,\delta) = L(\bx,\delta) + \Delta.
  \end{equation}
  Because $\theta^{\delta}(\bx(u),\bx(v)) = 0$ for all $(u,v)$, we have
  \[
  L(\bx,\delta) = \sum_u \theta_u^{\delta}(\bx(u)).
  \]
  Additionally, because $\theta^{\delta}_{uv}(i,j) \ge 0$,
  \[
  L(x,\delta) = \sum_u\sum_i\theta_u^{\delta}(i)x_u(i) + \sum_{uv}\sum_{ij}\theta^{\delta}_{uv}(i,j)x_{uv}(i,j) \ge \sum_u\sum_i\theta_u^{\delta}(i)x_u(i).
  \]
  Combining the above two inequalities with \eqref{eqn:lagrangian-ineq} gives:
  \begin{equation}\label{eqn:node-ineq}
  \sum_u\sum_i\theta_u^{\delta}(i)x_u(i)  \le \sum_u \theta_u^{\delta}(\bx(u)) + \Delta
  \end{equation}
  Because $\delta$ is locally decodable to $\bx$, and the smallest node margin is equal to $\psi(\delta)$, we have that for every $u$,
  $\theta^{\delta}_u(\bx(u)) + \psi(\delta) \le \theta^{\delta}_u(i)$ for all $i\ne \bx(u)$.
  The margin condition implies:
  \[
  \sum_u\theta^{\delta}_u(\bx(u))x_u(\bx(u)) + \sum_u\sum_{i\ne \bx(u)}(\theta^{\delta}_u(\bx(u)) + \psi(\delta))x_u(i) < \sum_u\sum_i\theta_u^{\delta}(i)x_u(i),
  \]
  and simplifying using $\sum_ix_u(i) = 1$ gives:
  \[
  \sum_u\theta^{\delta}_u(\bx(u)) + \psi(\delta)\sum_u\sum_{i\ne \bx(u)}x_u(i) < \sum_u\sum_i\theta_u^{\delta}(i)x_u(i).
  \]
  Plugging in to \eqref{eqn:node-ineq} gives:
  \[
  \sum_u\sum_{i\ne \bx(u)}x_u(i) < \frac{\Delta}{\psi(\delta)}.
  \]
  The left-hand-side is precisely $||x_V-\bx_V||_1 / 2$.
\end{proof}
Now we show that $(2,1,\psi)$-expansion stability implies that there exists a locally decodable dual solution $\delta$ with dual margin $\psi(\delta) \ge \psi$.

\begin{lemma}[$(2,1,\psi)$-expansion stability gives a lower bound on dual margin]
Let $(G,c,w)$ be a $(2,1,\psi)$-expansion stable instance with $\psi > 0$. Then there exists a locally decodable dual solution $\delta$ with dual margin $\psi(\delta) \ge \psi$.
\end{lemma}
\begin{proof}
Define new costs $c_{\psi}$ as
\[
c_\psi(u,i) = \begin{cases}
c(u,i) + \psi & \bx(u) = i\\
c(u,i) & \text{otherwise.}
\end{cases}
\]
By definition, the instance $(G,c_{\psi},w)$ is $(2,1)$-expansion stable (see Definition \ref{def:expansion-stability}). Theorem \ref{thm:lptes} implies the pairwise LP solution is unique and integral on $(G,c_{\psi},w)$. This implies there exists a dual solution $\delta^0$ that is locally decodable to $\bx$. The only guarantee on the dual margin of $\delta^0$ is that $\psi(\delta^0) > 0$. But note that $\delta^0$ is also an optimal dual solution for $(G,c,w)$, since its objective in that instance is the same as the objective of $\bx$. But in that instance, the dual margin at every node is at least $\psi$, because $c_{\psi}(u,\bx(u)) - c(u,\bx(u)) = \psi$. So $\psi(\delta) \ge \psi$.
\end{proof}

These two lemmas directly imply Theorem \ref{thm:curvature}. This dual proof is slightly more general than the primal proof, since the curvature result applies to any $x\in L(G)$.

\section{Details for Generative model}\label{sec:random-model_details}

\begin{definition}[sub-Gaussians and $(b,\sigma)$-truncated sub-Gaussians] \label{def:subG}
Suppose $b \in \R, \sigma \in \R_+$.
A random variable $X$ with mean $\mu$ is sub-Gaussian with parameter $\sigma$ if and only if  $\E[e^{\lambda(X-\mu)}]\le \exp(\lambda^2 \sigma^2/2)$ for all $\lambda \in \R$.
The random variable $X$ is $(b,\sigma)$-truncated sub-Gaussian if and only if $X$ is supported in $(b,\infty)$ and $X$ is sub-Gaussian with parameter $\sigma$. 
\end{definition}

We remark that the above definition captures many well-studied families of bounded random variables e.g., Rademacher distributions, uniform distributions on an interval etc. We remark that a bounded random variable supported on $[-M,M]$ is also sub-Gaussian with parameter $M$. However in our setting, it needs to be truncated only on negative side, and the bound $M$ will be much larger than the variance parameter $\sigma$; the bound is solely to ensure non-negativity of edge costs. A canonical example to keep in mind is a truncated Gaussian distribution.%\anote{What parameters?}  
We use the following standard large deviations bound for sums of sub-Gaussian random variables (for details, refer to Thm 2.6.2 from \citet{vershynin_hdp}). Given independent r.v.s $X_1, X_2, \dots, X_n$, with $X_i$ drawn from a sub-Gaussian with parameter $\sigma_i$ we have for $\mu=\sum_{i=1}^n \E[X_i]$ and $\sigma^2 = \sum_{i=1}^n \sigma_i^2$,
\begin{equation}\label{eq:deviationbound}
    \Pr\Big[ \Big| \sum_{i=1}^n X_i - \mu \Big| \ge t \Big] \le 2\exp\Big(-\frac{t^2}{2\sigma^2} \Big).
\end{equation} 

\dswhp*

\begin{proof}
As discussed in section~\ref{sec:expansion-stability_details}, for any $x \in L^*(G)$, the objective of the local LP can be written as \[\sum_{u\in V}\sum_i c(u,i)x_u(i) + \sum_{(u,v) \in E}w(u,v)d(u,v)\]
where $d(u,v) \coloneqq \frac{1}{2}\sum_{i}\abs{x_u(i) - x_v(i)}$. Let $\hat f(x) \coloneqq \dot{\hat{\theta}}{x}, \bar{f}(x) \coloneqq \dot{\bar{\theta}}{x}$. Then,

\begin{align*}
\abs{\dot{\hat{\theta}}{x}  - \dot{\bar{\theta}}{x}} &= \abs{\hat{f}(x) - \bar{f}(x)} = \Big|\sum_{u \in V} \sum_{i\in L} \tc(u,i) x_u(i) + \sum_{(u,v) \in E} \tw(u,v)d(u,v)\Big|
\end{align*}

For any feasible LP solution $x$, consider the following rounding algorithm $\calR$: 

\begin{algorithm}[H]
   \caption{$\calR$ rounding}
   \label{alg:R_rounding}
\begin{algorithmic}[1]
   \FOR{each $i \in \calL$}
        \STATE Choose $r_i \in (0,1)$ uniformly at random.
        \FOR{each $u \in V$}
            \IF {$x_u(i) > r_i$}
                \STATE $\calR(x)_u(i) = 1$.
            \ELSE
                \STATE $\calR(x)_u(i) = 0$
            \ENDIF 
        \ENDFOR
    \ENDFOR
\end{algorithmic}
\end{algorithm}

Then, we have
\begin{align*}
\E[\hat f(\calR(x)) - \bar{f}(\calR(x))] &= \sum_{u \in V} \sum_{i\in L} \tc(u,i) \E[\indicator[\calR(x)_u(i) = 1]] + \sum_{(u,v) \in E} \frac{\tw(u,v)}{2}\sum_{i} \E[\indicator[\calR(x)_u(i) \neq \calR(x)_v(i)]]
\\&= \sum_{u \in V} \sum_{i\in L} \tc(u,i) \Pr[x_u(i) > r_i] + \sum_{(u,v) \in E} \frac{\tw(u,v)}{2}\sum_{i} \Pr[\min\rbr{x_u(i), x_v(i)} \leq r_i < \max\rbr{x_u(i),x_v(i)}]
\\&= \sum_{u \in V} \sum_{i\in L} \tc(u,i) x_u(i) + \sum_{(u,v) \in E} \frac{\tw(u,v)}{2}\sum_{i} \abs{x_u(i) - x_v(i)}
\\&= \sum_{u \in V} \sum_{i\in L} \tc(u,i) x_u(i) + \sum_{(u,v) \in E} \tw(u,v)d(u,v) = \hat f(x) - \bar{f}(x) 
\end{align*}

Therefore,
\[\sup_{x \in L^*(G)}\abs{\hat f(x) - \bar{f}(x)} = \sup_{x \in L^*(G)} \abs{\E[\hat f(\calR(x)) - \bar{f}(\calR(x))] } \leq \sup_{\hat x_V \in \{0,1\}^{nk}} \abs{\hat{f}(\hat x_V) - \bar{f}(\hat x_V)}\]

Note that for all $x \in L^*(G)$, $\hat{f}(x)$ and $\bar{f}(x)$ only depend on the portion of $x$ restricted to the vertices i.e. $x_V$. This is why we only need to look at $\hat{x}_v \in \{0,1\}^{nk}$ for the last inequality.

For any fixed $\hat x_V \in \{0,1\}^{nk}$, since $\tw(u,v), \tc(u,i)$ are all mean $0$ and sub-Gaussian with parameters $\gamma_{u,v}, \sigma_{u,i}$, we have for any $t > 0$,
\[ \Pr \left[\abs{\hat f(\hat x_V) - \bar{f}(\hat x)} > t\right] \leq 2\exp\left(\frac{-t^2}{2\rbr{\sum_{u,i}\sigma_{u,i}^2 + k^2/4\sum_{uv}\gamma_{u,v}^2}}\right) \]

Taking $t = c\sqrt{nk}\sqrt{\sum_{u,i}\sigma_{u,i}^2 + k^2/4\sum_{uv}\gamma_{u,v}^2}$, we get that for any fixed $\hat x_V \in \{0,1\}^{nk}$,

\[ \Pr \left[\abs{\hat f(\hat x) - \bar{f}(\hat x)} > t\right] \leq 2\exp\left(-c^2 nk\right) \]

Taking a union bound over $\{0,1\}^{nk}$, we get that

\[ \Pr \left[ \sup_{\hat x_V \in \{0,1\}^{nk}}  \abs{\hat f(\hat x) - \bar{f}(\hat x)} > t\right] \leq 2\exp\left(nk\rbr{\log 2 - c^2}\right) \]
\end{proof}

Here, $c$ needs to be greater than $\sqrt{\ln 2} \approx 0.83$ to get a high probability guarantee.

\mapreg*
\begin{proof}
From Theorem~\ref{thm:apmap}, we have that, with high probability over the random noise \[ \frac{1}{2}\norm{\hx_V - \bar{x}_V}_1 \le \frac{2}{\psi} \cdot c\sqrt{nk} \cdot \sqrt{\sum_{u,i}\sigma_{u,i}^2 + \frac{k^2}{4}\sum_{uv}\gamma_{u,v}^2} \]
In this setting, this leads to 
\[ \frac{2}{\psi} \cdot c\sqrt{nk} \cdot \sqrt{\sum_{u,i}\sigma_{u,i}^2 + \frac{k^2}{4}\sum_{uv}\gamma_{u,v}^2} = \frac{2}{\psi} \cdot c\sqrt{nk} \cdot \sqrt{\rho nk\sigma^2 + \eta \frac{k^2}{4} \frac{nd}{2} \gamma^2}
    = \frac{2cnk\sqrt{\rho \sigma^2 + \frac{\eta dk}{8} \gamma^2}}{\psi} \]
since $\abs{V} = n, \abs{L} = k, \text{ and } \abs{E} = \frac{nd}{2}$.

\end{proof}

\section{Algorithm for finding nearby stable instances details }
\label{sec:algorithm_details}
Let $\bx$ be a MAP solution, and let $\calE^{\bx}$ be the set of expansions of $\bx$.
We prove that an instance is $(2,1,\psi)$-expansion stable if and only if $\dot{\theta_{adv}}{\bx} \le \dot{\theta_{adv}}{x}$ for all $x \in \calE^{\bx}$. In other words, it is sufficient to check for stability in the \emph{adversarial} perturbation for $\bx$. 
This proves that we need not check every possible perturbation when finding a $(2,1,\psi)$-expansion stable instance.
\begin{claim}\label{claim:adv-enough}
Let $(G,c,w)$ be an instance of uniform metric labeling with MAP solution $\bx$. Define:
\begin{equation*}
    w_{adv}(u,v) = \begin{cases}
    \frac{1}{2}w_{uv} & \bx(u) = \bx(v)\\
    w(u,v) & \bx(u) \ne \bx(v)
    \end{cases}
\end{equation*}
\begin{equation*}
    c_{adv}(u,i) = \begin{cases}
        c(u,i) + \psi & \bx(u) = i\\
        c(u,i) & \text{otherwise.}
    \end{cases}
\end{equation*}
Let $\theta_{adv}$ be the objective vector in the instance $(G,c_{adv},w_{adv})$. Then
\[
\dot{\theta'}{\bx} \le \dot{\theta'}{x}
\]
for all $(2,1,\psi)$-perturbations $\theta'$ of $\theta$ and all $x\in \calE^{\bx}$ if and only if:
\[
\dot{\theta_{adv}}{\bx} \le \dot{\theta_{adv}}{x}
\]
for all $x\in \calE^{\bx}$.
\end{claim}
\begin{proof}
    The proof is analogous to that of Claim \ref{exp_worst}. If the instance is $(2,1,\psi)$-expansion stable, then $\dot{\theta'}{x} - \dot{\theta'}{\bx} \ge 0$ for all $(2,1,\psi)$-perturbations $\theta'$ and all expansions $x$ of $\bx$. Because $\theta_{adv}$ is a valid $(2,1,\psi)$-perturbation, this gives one direction. For the other, note that if the instance is not $(2,1,\psi)$-expansion stable, there exists a $\theta'$ and an $x\in\calE^{\bx}$ for which $\dot{\theta'}{x} - \dot{\theta'}{\bx} < 0$. A direct computation shows that $\dot{\theta'}{x} - \dot{\theta'}{\bx} \ge \dot{\theta_{adv}}{x} - \dot{\theta_{adv}}{\bx}$ for all $(2,1,\psi)$-perturbations $\theta'$ of $\theta$. Then we have $\dot{\theta_{adv}}{x} - \dot{\theta_{adv}}{\bx} < 0$.
\end{proof}
This claim justifies \eqref{eqn:alg}, which only enforces that $\bx$ is at least as good as all of its expansions in $\theta_{adv}$. The following claim implies that there is always a feasible point of \eqref{eqn:alg} that makes modifications of bounded size to $c$ and $w$.
\begin{claim}
Consider an instance $(G,c,w)$ with a unique MAP solution $\bx$. Let $w'$ be defined as
\begin{equation*}
    w'(u,v) = \begin{cases}
    w(u,v) & \bx(u) \ne \bx(v)\\
    2w(u,v) & \bx(u) = \bx(v),
    \end{cases}
\end{equation*}
and let $c'$ be defined as
\begin{equation*}
c'(u,i) = \begin{cases}c(u,i) - \psi & \bx(u) = i\\
c(u,i) & \bx(u) \ne i.\end{cases}
\end{equation*}
Then the instance $(G,c',w')$ is $(2,1,\psi)$-expansion stable with MAP solution $\bx$.
\end{claim}
\begin{proof}[Proof (sketch)]
The original MAP solution $\bx$ is also the MAP solution to $(G,c',w')$. Then the original instance $(G,c,w)$ is obtained from $(G,c',w')$ by performing the adversarial $(2,1,\psi)$-perturbation for $\bx$ (see Claim \ref{claim:adv-enough}). Because $\bx$ was the unique MAP solution to this instance, it has better objective than all of its expansions. Therefore, $(G,c',w')$ is $(2,1,\psi)$-expansion stable, by Claim \ref{claim:adv-enough}.
\end{proof}
$(G,c',w')$ is a ``nearby'' stable instance to $(G,c,w)$, but it requires changes to quite a few edges---every edge that is not cut by $\bx$---and changes the node costs of every vertex. Surprisingly, the stable instances we found in Section \ref{sec:experiments} were much closer than $(G,c',w')$---that is, only sparse changes were required to transform the observed instance $(G,c,w)$ to a $(2,1,\psi)$-expansion stable instance.

\section{Experiment details}\label{sec:experiments_details}
In this section, we give more details for the numerical examples for which we evaluate our curvature bound from Theorem \ref{thm:curvature}. We studied instances for stereo vision, where the input is two images taken from slightly offset locations, and the desired output is the disparity of each pixel location between the two images (this disparity is inversely proportional to the depth of that pixel). We used the models from \citet{tappen2003comparison} on three images from the Middlebury stereo dataset \citep{scharstein2002taxonomy}. In this model, $G$ is a grid graph with one node corresponding to each pixel in one of the images (say, the one taken from the left), the costs $c(u,i)$ are set using the Birchfield-Tomasi matching costs \citep{birchfield1998pixel}, and the edge weights $w(u,v)$ are set as:
\[
w(u,v) = \begin{cases}
P \times s & |I(u) - I(v)| < T\\
s & \text{otherwise}.
\end{cases}
\]
Here $I(u)$ is the intensity of one of the images (again, say the left image) at pixel location $u$, and we set $(P,T,s) = (2,50,4)$. This is a Potts model. The {\tt tsukuba}, {\tt cones}, and {\tt venus} images were {\tt 120 x 150}, {\tt 125 x 150}, and {\tt 383 x 434}, respectively. These models had $k=7$, $k=5$, and $k=5$, respectively.

To generate Table \ref{tbl:boundtable}, we ran the algorithm in \eqref{eqn:alg} using Gurobi \citep{gurobi} for the L1 distance. For each observed instance $(G,\hc,\hw)$, this output a nearby $(2,1,\psi)$-stable instance $(G,\bc,\bw)$. In all of our experiments, we used $\psi=1$. Additionally, we always set the target MAP solution $x^t$ in \eqref{eqn:alg} to be equal to the observed MAP solution $\hx^{MAP}$. To evaluate our recovery bound, we compared the objective of the observed LP solution $\hat{x}$ to the $\hx^{MAP}$ in $(G,\bc,\bw)$. That is, if $\bar{\theta}$ is the objective for $(G,\bc,\bw)$, we computed $\dot{\bar{\theta}}{\hx} - \dot{\bar{\theta}}{\hx^{MAP}} = \dot{\bar{\theta}}{\hx} - \dot{\bar{\theta}}{\bx}$, where the second equality is because we set the target solution $x^t$ to be equal to $\hx^{MAP}$, so $\hx^{MAP} = \bx$. Because $\psi=1$, the difference between these two objectives is precisely the value of our curvature bound. In particular, Theorem \ref{thm:curvature} guarantees that \[
\frac{1}{2n}||\hx_V - \hx^{MAP}_V||_1 \le \frac{1}{n}\left(\dot{\bar{\theta}}{\hx} - \dot{\bar{\theta}}{\hx^{MAP}}\right).
\]
The right-hand-side is shown for these instances in the ``Recovery error bound'' column of Table \ref{tbl:boundtable}, and the true value of $\frac{1}{2n}||\hx_V - \hx^{MAP}_V||_1$ (i.e., the true recovery error) is shown in the identically titled column of Table \ref{tbl:boundtable}.
On these instances, $\frac{1}{n}\left(\dot{\bar{\theta}}{\hx} - \dot{\bar{\theta}}{\hx^{MAP}}\right)$ is close to 0, so our curvature bound ``explains'' a large portion of $\hx$'s recovery of $\hx^{MAP}$. These instances are close to $(2,1,\psi)$-stable instances where $\hx$ and $\hx^{MAP}$ have close objective, and this implies by Theorem \ref{thm:curvature} that $\hx$ approximately recovers $\hx^{MAP}$. 

However, this result relies on a property of the LP solution $\hx$: that it has good objective in the stable instance discovered by the procedure \eqref{eqn:alg}. Compare this to Corollary \ref{cor:deviation}, which only depends on properties of the observed instance $\hat{\theta}$ and the stable instance $\bar{\theta}$ (in particular, some notion of ``distance'' between them). Given an observed instance $\hat{\theta}$ and stable instance $\bar{\theta}$, we can try to compute $d(\bar{\theta}, \hat{\theta})$ from Corollary \ref{cor:deviation} to give a bound that does not depend on $\hat{x}$. Unfortunately, this distance can be large, leading to a bound that can be vacuous (i.e., normalized Hamming recovery $> 1$). The following refinement of Corollary \ref{cor:deviation} gives much tighter bounds.

\begin{restatable}[LP solution is good if there is a nearby stable instance, refined]{corollary}{deviation2}\label{cor:deviation2}
Let $\hx^{MAP}$ and $\hx$ be the MAP and local LP solutions to an observed instance $\obsins$. Also, let $\bar{x}$ be the MAP solution for a latent $(2,1,\psi)$-expansion stable instance $\stabins$. If $\hat{\theta} = (\hat{c},\hat{w})$ and $\bar{\theta} = (\bar{c}, \bar{w})$, define
\[
d(\bar{\theta}, \hat{\theta}) \coloneqq \sup_{x \in L^*(G)} \dot{\bar{\theta}}{x}  - \dot{\bar{\theta}}{\bx} - (\dot{\hat{\theta}}{x}  - \dot{\hat{\theta}}{\bx}).
\]
Note that while we still use the name $d(\cdot, \cdot)$, evoking a metric, $d$ is not symmetric. Then
\[ \frac{1}{2}\norm{\hx_V - \hx^{MAP}_V}_1 \le \frac{d(\bar{\theta}, \hat{\theta})}{\psi} + \frac{1}{2}\norm{\hx_V^{MAP} - \bx_V}_1. \]
\end{restatable}
\begin{proof}
Note:
\begin{align*}
    \frac{1}{2}\norm{\hx_V - \hx_V^{MAP}}_1 &\leq \frac{1}{2}\norm{\hx_V - \bx_V}_1 + \frac{1}{2}\norm{\hx_V^{MAP} - \bx_V}_1.
\end{align*}
By the definition of $d$, for any $x \in L^*(G)$, 
\[\dot{\bar{\theta}}{x}  - \dot{\bar{\theta}}{\bx} \le d(\bar{\theta},\hat{\theta}) + (\dot{\hat{\theta}}{x}  - \dot{\hat{\theta}}{\bx}). \]
Now if we set $x = \hx$, the LP solution to the observed instance, we have $\dot{\hat{\theta}}{\hx}  - \dot{\hat{\theta}}{\bx}) \le 0$, so 
\[
\dot{\bar{\theta}}{\hx}  - \dot{\bar{\theta}}{\bx} \le d(\bar{\theta},\hat{\theta}).
\]
Theorem \ref{thm:curvature} then implies $\frac{1}{2}\norm{\hx_V - \bx_V}_1 \le d(\bar{\theta}, \hat{\theta})/\psi$, which gives the claim.
\end{proof}
Given an observed instance $\hat{\theta}$ and a $(2,1,\psi)$-expansion stable instance $\bar{\theta}$ output by \eqref{eqn:alg} with $\bar{x}^{MAP}$ = $\hat{x}^{MAP}$, we can upper bound $d(\bar{\theta}, \hat{\theta})$ by computing 
\[
d_{up}(\bar{\theta}, \hat{\theta}) \coloneqq \sup_{x \in L(G)} \dot{\bar{\theta}}{x}  - \dot{\bar{\theta}}{\bx} - (\dot{\hat{\theta}}{x}  - \dot{\hat{\theta}}{\bx}),
\]
which is a linear program in $x$ because we relaxed $L^*(G)$ to $L(G)$. Corollary \ref{cor:deviation2} then implies that the recovery error of $\hat{x}$ is at most $d_{up}(\bar{\theta},\hat{\theta}) / \psi$, which we can compute. Table \ref{tbl:boundtable2} shows the results of this procedure on two of the same instances from Table \ref{tbl:boundtable} in the ``Unconditional bound'' column. While the values of this bound are much larger than the ``Curvature bound'' of Theorem 5.2, they are much more theoretically appealing, since they only depend on the difference between $\hat{\theta}$ and $\bar{\theta}$ rather than on a property of the LP solution $\hat{x}$ to $\hat{\theta}$. For Table \ref{tbl:boundtable2}, we did a grid search for $\psi$ over $\{1,\ldots, 10\}$; $\psi=4$ gave the optimal unconditional bound for both instances. The difference in $\psi$ explains the  slight differences between the other columns of Tables \ref{tbl:boundtable} and \ref{tbl:boundtable2}.

\begin{table*}[ht]
     \centering
     \caption{Results from the output of \eqref{eqn:alg} on two stereo vision instances. Curvature bound shows the bound obtained from Theorem \ref{thm:curvature}, which depends on the observed LP solution $\hat{x}$. Unconditional bound shows the bound from the refined version of Corollary \ref{cor:deviation}, which depends \emph{only} on the observed instance and the stable instance. This ``unconditional'' bound explains a reasonably large fraction of the LP solution's recovery for these instances: because the instance is close to a stable instance, the LP solution approximate recovers the MAP solution.}
     \begin{tabular}{lccccc}
          Instance & Costs changed & Weights changed & Curvature bound & Unconditional bound & $||\hat{x}_V - \hat{x}^{MAP}_V||_1/2n$ \\
          \toprule
          ${\tt tsukuba}$ & 4.9\% & 2.8\% & 0.0173 & 0.4878 & 0.0027\\
          ${\tt cones}$ & 2.81\% & 2.31\% & 0.0137 & 0.2819 & 0.0022\\
          \bottomrule
     \end{tabular}
     \label{tbl:boundtable2}
 \end{table*}

\clearpage

\end{document}